\pdfoutput=1
\documentclass{article}

\usepackage[square,sort,comma,numbers]{natbib}
\usepackage[final]{arxiv}

\usepackage[utf8]{inputenc} 
\usepackage[T1]{fontenc}    
\usepackage{hyperref}       
\usepackage{url}            
\usepackage{booktabs}       
\usepackage{amsfonts}       
\usepackage{nicefrac}       
\usepackage{microtype}      
\usepackage{amsmath}
\usepackage{xcolor}

\usepackage{amssymb}
\usepackage{amsthm}
\usepackage{subfigure}
\usepackage{graphicx}
\usepackage{comment}

\usepackage[ruled]{algorithm2e}
\usepackage{amsmath, amsthm, thm-restate}
\usepackage{chngcntr}

\newtheorem{definition}{Definition}
\newtheorem{lemma}{Lemma}

\newtheorem{theorem}{Theorem}

\title{Demystifying Orthogonal Monte Carlo and Beyond}

\author{%
  Han Lin \thanks{equal contribution} \\
  Columbia University\\
  \And
  Haoxian Chen \footnotemark[1] \\
  Columbia University\\
  \And
  Tianyi Zhang \\
  Columbia University\\
  \And
  Clement Laroche \\
  Columbia University\\
  \And
  Krzysztof Choromanski \\
  Google Brain Robotics \& Columbia University \\
}

\begin{document}

\maketitle

\begin{abstract}
  Orthogonal Monte Carlo \cite{ort} (OMC) is a very effective sampling algorithm imposing structural geometric conditions (orthogonality) on samples for variance reduction. Due to its simplicity and superior performance as compared to its Quasi Monte Carlo counterparts, OMC is used in a wide spectrum of challenging machine learning applications ranging from scalable kernel methods \cite{unreas} to predictive recurrent neural networks \cite{psrnn}, generative models \cite{hron}  and reinforcement learning \cite{es_ort}.
  However theoretical understanding of the method remains very limited. In this paper we shed new light on the theoretical principles behind OMC, applying theory of negatively dependent random variables to obtain several new concentration results. 
  As a corollary, we manage to obtain first uniform convergence results for OMCs and consequently, substantially strengthen best known downstream guarantees for kernel ridge regresssion via OMCs. We also propose a novel extensions of the method leveraging theory of algebraic varieties over finite fields and particle algorithms, called \textit{Near-Orthogonal Monte Carlo} (NOMC). We show that NOMC is the first algorithm consistently outperforming OMC in applications ranging from kernel methods to approximating distances in probabilistic metric spaces.

\end{abstract}

\section{Introduction \& Related Work}
\label{sec:intro_rel_work}

Monte Carlo (MC) methods are widely applied in machine learning in such domains as: dimensionality reduction \cite{jlt,jlt2,jlt3}, scalable kernel methods with random feature maps \cite{rahimi}, generative modeling and variational autoencoders via sliced Wasserstein distances \cite{hron},  approximating Gaussian smoothings in Evolutionary Strategies (ES) algorithms for Reinforcement Learning (RL) \cite{es_ort}, predictive recurrent neural networks \cite{psrnn} and more.
The theory of MC is rich with various techniques improving the accuracy of base MC estimators such as: antithetic couplings and importance sampling \cite{mc_var_1}, variance reduction via carefully designed control variate terms \cite{mc_var_2, mc_var_3} and finally: the vast field of the so-called \textit{Quasi Monte Carlo} (QMC) methods \cite{qmc_1, qmc_2, qmc_3, qmc_4}. 

Relatively recently, yet another algorithm which can be combined with most of the aforementioned approaches, called \textit{Orthogonal Monte Carlo} (OMC) has been proposed \cite{ort}. OMC relies on ensembles of mutually orthogonal random samples for variance reduction and turns out to be very effective in virtually all applications of MC in machine learning involving isotropic distributions \cite{hron,psrnn,unreas, uni, geo}. Providing substantial accuracy improvements over MC baselines, conceptually simple, and superior to algorithms leveraging QMC techniques, it became one of the most frequently used techniques in a vast arsenal of MC tools.

OMCs are also much simpler than the class of MC methods based on determinantal point processes (DPPs) \cite{dpp_main}. DPPs provide elegant mechanisms for sampling diverse ensembles, where diversity is encoded by a kernel.
Some DPP-MCs \cite{guillaume, choro_dpp, bardenet2016monte} provide stronger theoretical guarantees than base MCs, yet those are impractical to use in higher dimensions due to their prohibitive time complexity, especially when samples need to be frequently constructed such as in RL (\cite{es_ort}).

Despite its effectiveness and impact across the field, theoretical principles behind the OMC method remain only partially understood, with theoretical guarantees heavily customized to specific applications and hard to generalize to other settings \cite{geo, uni, choro_sind}.

In this paper we shed new light on the effectiveness of OMCs by applying theory of \textit{negatively dependent} random variables that is a theoretical backbone of DPPs. Consequently, we present first comprehensive theoretical view on OMCs. Among our new results are first exponentially small probability bounds for errors of  OMCs applied to objectives involving \textbf{general nonlinear mappings}. Previously such results were known only for the cosine mapping in the setting of Gaussian kernel approximation via random features \cite{psrnn} and for random linear projections for dimensionality reduction. 
Understanding the effectiveness of OMCs in the general nonlinear setting was considered the Holy Grail of the research on structured MC methods, with elusive general theory. This striking discrepancy between practice where OMCs are used on a regular basis in general nonlinear settings and very limited developed theory is one of the main motivations of this work.

Our techniques enable us to settle several open conjectures for OMCs. Those involve not only aforementioned results for the general nonlinear case, but strong concentration results for arbitrary RBF kernels with no additional assumptions regarding corresponding spectral densities, in particular first such results for all Mat\'{e}rn kernels. 
We show that our concentration results directly imply uniform convergence of OMCs (which was an open question) and that these lead to substantial strengthening of the best known results for kernel ridge regression via OMCs from \cite{psrnn}. The strengthenings are twofold: we extend the scope to all RBF kernels as opposed to just \textit{smooth RBFs} \cite{psrnn} and we significantly improve accuracy guarantees.

One of the weaknesses of OMCs is that orthogonal ensembles can be defined only if the number of samples $s$ satisfies $s \leq d$, where $d$ stands for data dimensionality. 
In such a setting a relaxed version of the method is applied, where
one orthogonal block is replaced by multiple independent orthogonal blocks \cite{ort}.
Even though orthogonal entanglement of samples across different blocks is now broken, such block-orthogonal OMC methods (or B-OMCs) were still the most accurate known MC algorithms for isotropic distributions when $s \gg d$.

We propose an extension of OMCs relying on the ensembles of random near-orthogonal vectors preserving entangelements across all the samples, called by us \textit{Near-Orthogonal Monte Carlo} (NOMC), that to the best of our knowledge, is the first algorithm beating B-OMCs. We demonstrate it in different settings such as: kernel approximation methods and approximating \textit{sliced Wasserstein distances} (used on a regular basis in generative modeling). NOMCs are based on two new paradigms for constructing structured MC samples: high-dimensional optimization with particle methods and the theory of algebraic varieties over finite fields.

We highlight main contributions below.
Conclusions and broader impact analysis is given in Sec. \ref{sec:broader_impact}.
\vspace{-2mm}

\begin{itemize}
    \item By leveraging the theory of negatively dependent random variables, we provide first exponentially small bounds on error probabilities for OMCs used to approximate objectives involving \textbf{general nonlinear mappings} [Sec. \ref{sec:theory}: Theorem \ref{first_thm}, Theorem \ref{Thm_StrongConcentration_main}]. 
    \item We show how our general theory can be used to obtain simpler proofs of several known results and new results not known before [Sec. \ref{sec:tech_intro}, Sec. \ref{sec:applications}], e.g. first Chernoff-like concentration inequalities regarding certain classes of Pointwise Nonlinear Gaussian (PNG) kernels and all RBF kernels (previously such results were known only for RBF kernels with corresponding isotropic distributions of no heavy tails \cite{geom, psrnn}). 
    \item Consequently, we provide first uniform convergence results for OMCs and as a corollary, apply them to obtain new SOTA downstream guarantees for kernel ridge regression with OMCs [Sec. \ref{sec:uniform_convergence}], improving both: accuracy and scope of applicability.
    \item We propose two new paradigms for constructing structured samples for MC methods when $s \gg d$, leveraging number theory techniques and particle methods for high-dimensional optimization. In particular, we apply a celebrated Weil Theorem \cite{weil} regarding generating functions derived from counting the number of points on algebraic varieties over finite fields.
    \item We empirically demonstrate the effectiveness of NOMCs [Sec. \ref{sec:experiments}].
\end{itemize}
\vspace{-2mm}

\vspace{-3mm}
\section{Orthogonal Monte Carlo}
\label{sec:tech_intro}
\vspace{-2mm}
Consider a function $f_{\mathcal{Z}}: \mathbb{R}^{d} \rightarrow \mathbb{R}^{k}$, parameterized by an ordered subset $\mathcal{Z} \subseteq_{\mathrm{ord}} \mathcal{P}(\mathbb{R}^{d})$ and let:
\begin{equation}
\label{main_eq}
F_{f,\mathcal{D}}(\mathcal{Z}) \overset{\mathrm{def}}{=} \mathbb{E}_{\omega \sim \mathcal{D}} [f_{\mathcal{Z}}(\omega)],   
\end{equation}
where $\mathcal{D}$ is an isotropic probabilistic distribution on $\mathbb{R}^{d}$. In this work we analyze MC-based approximation of $F$. Examples of important machine learning instantiations of $F$ are given below.
\vspace{-6.5mm}
\paragraph{Kernel Functions \& Random Features:} Every shift-invariant kernel $K:\mathbb{R}^{d} \times \mathbb{R}^{d} \rightarrow \mathbb{R}$ can be written as $K(\mathbf{x},\mathbf{y})=g(\mathbf{x}-\mathbf{y})
\overset{\mathrm{def}}{=}\mathbb{E}_{\omega \sim \mathcal{D}}[\cos(\omega^{\top}(\mathbf{x}-\mathbf{y}))]$
for some probabilistic distribution $\mathcal{D}$ \cite{rahimi}. Furthermore, if $K$ is a \textit{radial basis function} ($\mathrm{RBF}$) kernel (e.g. Gaussian or Mat\'{e}rn), i.e. $K(\mathbf{x},\mathbf{y}) = r(\|\mathbf{x}-\mathbf{y}\|_{2})$ for some $r:\mathbb{R}_{\geq 0} \rightarrow \mathbb{R}$, then $\mathcal{D}$ is isotropic. Here $\mathcal{Z} = (\mathbf{z})$, where $\mathbf{z} = \mathbf{x} - \mathbf{y}$, and $f_{(\mathbf{z})}(\omega)=\cos(\omega^{\top}\mathbf{z})$.
For \textit{pointwise nonlinear Gaussian} [$\mathrm{PNG}$] kernels \cite{unreas} (e.g. angular or arc-cosine), given as $K_{h}(\mathbf{x},\mathbf{y})=\mathbb{E}_{\omega \sim \mathcal{N}(0,\mathbf{I}_{d})}[h(\omega^{\top}\mathbf{x})h(\omega^{\top}\mathbf{y})]$, where $h:\mathbb{R} \rightarrow \mathbb{R}$, the corresponding distribution $\mathcal{D}$ is multivariate Gaussian and $f$ is given as $f_{(\mathbf{x},\mathbf{y})}=h(\omega^{\top}\mathbf{x})h(\omega^{\top}\mathbf{y})$.
\vspace{-3mm}
\paragraph{Dimensionality Reduction [JLT]:}
Johnson-Lindenstrauss dimensionality reduction techniques (JLT) \cite{jlt_base, jlt, matousek} rely on embeddings of high-dimensional feature vectors via random projections given by vectors $\omega \sim \mathcal{N}(0,\mathbf{I}_{d})$. Expected squared distances between such embeddings of input high-dimensional vectors $\mathbf{x},\mathbf{y} \in \mathbb{R}^{d}$ are given as: $\mathrm{dist}^{2}_{\mathrm{JLT}}(\mathbf{x},\mathbf{y})=\mathbb{E}_{\omega \sim \mathcal{N}(0,\mathbf{I}_{d})}[(\omega^{\top}(\mathbf{x}-\mathbf{y}))^{2}]$. Here $\mathcal{D}$ is multivariate Gaussian and $f_{(\mathbf{z})} = (\omega^{\top}\mathbf{z})^{2}$ for $\mathbf{z}=\mathbf{x}-\mathbf{y}$.
\vspace{-3mm}
\paragraph{Sliced Wasserstein Distances [SWD]:} Wasserstein Distances (WDs) are metrics in spaces of probabilistic distributions that have found several applications in deep generative models \cite{arjovsky,gulrajani}. For $p \geq 1$, the $p$-th Wasserstein distance between two distributions $\eta$ and $\mu$ over $\mathbb{R}^d$ is defined as:
$$\mathrm{WD}_p(\eta,\mu) = \left(\inf_{\gamma \in \Gamma(\eta,\mu)} \int_{\mathbb{R}^d \times \mathbb{R}^d} ||\mathbf{x}-\mathbf{y}||_{2}^{p} d\gamma(\mathbf{x},\mathbf{y}) \right)^{{\frac{1}{p}}},$$
where $\Gamma(\eta,\mu)$ is the set of joint distributions over $\mathbb{R}^d \times \mathbb{R}^d$ for which the marginal of the first/last $d$ coordinates is $\eta$/$\mu$. Since WD computations involve solving nontrivial optimal transport problem (OPT) \cite{opt} in the high-dimensional space, in practice its more efficient to compute proxies are used, among them the so-called \textit{Sliced Wasserstein Distance} (SWD) \cite{swd}. 
SWDs are obtained by constructing projections $\eta_{\mathbf{u}}$ and $\mu_{\mathbf{u}}$ of $\eta$ and $\mu$ into a random 1d-subspace encoded by $\mathbf{u} \sim \mathrm{Unif}(\mathcal{S}^{d-1})$ chosen uniformly at random from the unit sphere $\mathcal{S}^{d-1}$ in $\mathbb{R}^{d}$ (see: Sec. \ref{sec:experiments}). If $\eta$ and $\mu$ are given as point clouds, they can be rewritten as in Equation \ref{main_eq}, where $\mathcal{Z}$ encodes $\eta$ and $\mu$ via cloud points.
\vspace{-3mm}

\subsection{Structured ensembles for Monte Carlo approximation}

A naive way of estimating function $F_{f,\mathcal{D}}(\mathcal{Z})$ from Equation \ref{main_eq} is to generate $s$ independent samples : $\omega^{\mathrm{iid}}_{1},...,\omega^{\mathrm{iid}}_{s} \overset{\mathrm{iid}}{\sim} \mathcal{D}$, which leads to the base unbiased Monte Carlo (MC) estimator:
\vspace{-2mm}
\begin{equation}
\label{base_mc_eq}
\widehat{F}^{\mathrm{iid}}_{f,\mathcal{D}}(\mathcal{Z}) \overset{\mathrm{def}}{=} 
\frac{1}{s}\sum_{i=1}^{s}f_{\mathcal{Z}}(\omega_{i}^{\mathrm{iid}}).
\end{equation}
\vspace{-3mm}

Orthogonal Monte Carlo (OMC) method relies on the isotropicity of $\mathcal{D}$ and instead entangles different samples in such a way that they are $\textbf{exactly orthogonal}$, while their marginal distributions match those of $\omega_{i}^{\mathrm{iid}}$ (this can be easily done for instance via Gram-Schmidt orthogonalization followed by row-renormalization, see: \cite{ort}). Such an ensemble $\{\omega^{\mathrm{ort}}_{1},...,\omega^{\mathrm{ort}}_{s}\}$ is then used to
replace $\{\omega^{\mathrm{iid}}_{1},...,\omega^{\mathrm{iid}}_{s}\}$
in Equation \ref{base_mc_eq} to get OMC estimator $\widehat{F}^{\mathrm{ort}}_{f,\mathcal{D}}(\mathcal{Z})$.

Estimator $\widehat{F}^{\mathrm{ort}}_{f,\mathcal{D}}(\mathcal{Z})$ can be constructed only if $s \leq d$, where $d$ stands for samples' dimensionality. In most practical applications we have: $s > d$ and thus instead the so-called \textit{block orthogonal Monte Carlo} (B-OMC) procedure is used, where $s$ samples are partitioned into $d$-size blocks, samples within each block are chosen as above and different blocks are constructed independently \cite{ort}. In B-OMC, orthogonality is preserved locally within a block, but this entanglement is lost across the blocks.

In the next section we provide new general theoretical results for OMCs.

\vspace{-3mm}
\section{Orthogonal Monte Carlo and Negatively Dependent Ensembles}
\label{sec:theory}

For a rigorous analysis, we will consider an instantiation of the objective from Eq. \ref{main_eq} of the form:
\begin{equation}
\label{imp_eq_main}
F_{f,\mathcal{D}}(\mathbf{z}) = \mathbb{E}_{\omega \sim \mathcal{D}}[f(\omega^{\top}\mathbf{z})]    
\end{equation}
for $\mathbf{z} \in \mathbb{R}^{d}$ and some function $f:\mathbb{R} \rightarrow \mathbb{R}$.
We consider the following classes of functions $f(u)$:
\begin{enumerate}
    \item [\textbf{F1.}] monotone increasing or decreasing in $|u|$,
    \item [\textbf{F2.}] decomposable as $f=f^{+}+f^{-}$, where $f^{+}$ is monotone increasing and $f^{-}$ is monotone decreasing in $|u|$,
    \item [\textbf{F3.}] entire (i.e. expressible as a Taylor series with an infinite radius of convergence, e.g. polynomials).
\end{enumerate}

\textbf{Remark:} As we will see later, for the class \textbf{F3} the role of $f^{+}$ and $f^{-}$ in the analysis is taken by functions: $\mathrm{even}[f]^{+}$ and $\mathrm{even}[f]^{-}$, where $\mathrm{even}[f]$ stands for function obtained from $f$ by taking terms of the Taylor series expansion corresponding to even powers. 

Such objectives $F_{f,\mathcal{D}}$ are general enough to cover: dimensionality reduction setting, all RBF kernels, certain classes of PNG kernels and several statistics regarding neural network with random weights (see: Sec. \ref{sec:applications}) that we mentioned before. See also Table 1, where we give an overview of specific examples of functions covered by us, and Sec. \ref{sec:applications} for much more detailed analysis of applications.

For a random variable $X$ we define moment generating function $M_{X}$ as:
$M_{X}(\theta)=\mathbb{E}[e^{\theta X}]$. Furthermore, we define \textit{Legendre symbol} as: 
$\mathcal{L}_{X}(a) = \sup_{\theta>0} \log(\frac{e^{\theta a}}{M_{X}(\theta)})$ if $a > \mathbb{E}[X]$ and $\mathcal{L}_{X}(a) = \sup_{\theta<0} \log(\frac{e^{\theta a}}{M_{X}(\theta)})$ if $a < \mathbb{E}[X]$. It is a standard fact from probability theory that $\mathcal{L}_{X}(a) > 0$ for every $a \neq \mathbb{E}[X]$.

We prove first exponentially small bounds for failure probabilities of OMCs applied to functions from all three classes and in addition show that for the class \textbf{F1} obtained concentration bounds are better than for the base MC estimator.

Our results can be straightforwardly extended to classes of functions expressible as limits of functions from the above \textbf{F1}-\textbf{F3}, but for the clarity of the exposition we skip this analysis.
To the best of our knowledge, we are the first to provide theory that addresses also \textbf{discontinuous} functions.

\small
\vspace{-1mm}
\begin{center}
 \begin{tabular}{||c | c| c | c | c||} 
 \hline
 \null & $\textrm{JLT}$ & $\textrm{PNG: $h(x)=e^{cx}$}$ & $\textrm{Gaussian}$ &  $\nu$-\textrm{\textrm{Mat\'{e}rn}} \\ [0.5ex] 
 \hline\hline
 $\mathrm{class}$ & $\textrm{\textbf{F1},\textbf{F3}}$ &
 $\textbf{F3}$ & $\textrm{\textbf{F2},\textbf{F3}}$ &  $\textrm{\textbf{F2},\textbf{F3}}$ \\ 
 \hline
  $f^{+}(u)/\mathrm{even}[f]^{+}(u)$ & $x^{2}$ & $\sum_{k=0}^{\infty} \frac{(cu)^{2k}}{(2k)!}$ & $\sum_{k=0}^{\infty} \frac{u^{4k}}{(4k)!}$ &  $\sum_{k=0}^{\infty} \frac{u^{4k}}{(4k)!}$ \\ 
 \hline  
  $f^{-}(u)/\mathrm{even}[f]^{-}(u)$ & $\textrm{N/A}$ & $\textrm{N/A}$ & $-\sum_{k=0}^{\infty} \frac{u^{4k+2}}{(4k+2)!}$ &  $-\sum_{k=0}^{\infty} \frac{u^{4k+2}}{(4k+2)!}$ \\ 
 \hline  
 \textrm{SOTA results for OMC} & \textrm{ortho-JLTs} \cite{jlt_base} & \textrm{\textbf{ours}} & \textrm{\cite{psrnn}} & $\textrm{\textbf{ours}: any $\nu$}$ \\ 
 \hline 
\end{tabular}
\end{center}
\vspace{-1mm}
\small Table 1: Examples of particular instantiations of function classes \textbf{F1-F3} covered by our theoretical results.
\normalsize

The key tool we apply to obtain our general theoretical results is the notion of \textit{negative dependence} [ND] \cite{neg_dep_2, neg_dep_1}  that is also used in the theory of Determinantal Point Processes (DPPs) \cite{dpp_main}:

\begin{definition}[Negative Dependence (ND)] \label{Def_ND}
Random variables $X_1,…,X_n$ are said to be negatively dependent if both of the following two inequalities hold for any $x_1,…,x_n \in \mathbb{R}$

$$ \mathbb{P}( \bigcap_i X_i \geq x_i ) \leq \prod_i \mathbb{P}(X_i \geq x_i),\textrm{and }\mathbb{P}( \bigcap_i X_i \leq x_i ) \leq \prod_i \mathbb{P}(X_i \leq x_i).$$
\end{definition}

We show that certain classes of random variables built on orthogonal ensembles satisfy ND property:

\begin{restatable}[ND for OMC-samples and monotone functions]{lemma}{FirstLemma}
\label{nd-lemma}
For an isotropic distribution $\mathcal{D}$ on $\mathbb{R}^{d}$ and 
orthogonal ensemble: $\omega^{\mathrm{ort}}_{1},...,\omega^{\mathrm{ort}}_{d}$ with $\omega^{\mathrm{ort}}_{i} \sim \mathcal{D}$,
random variables:$X_{1},...,X_{d}$ defined as: $X_{i} = |\omega^{\mathrm{ort}}_{i}\mathbf{z}|$ are negatively dependent for any fixed $\mathbf{z} \in \mathbb{R}^{d}$.
\end{restatable}
 
Lemma 1 itself does not guarantee strong convergence for orthogonal ensembles however is one of the key technical ingredients that helps us to achieve this goal. The following is true:

\begin{restatable}{lemma}{SecondLemma}
\label{Lem_NA_ExpectationProductInequality}
Assume that $f$ is a function from the class $\textbf{F1}$.
Let $X_i = f(\omega_{i}^{\mathrm{ort}}\mathbf{z})$ for $i=1,...,n$, and let $\lambda$ be a non-positive (or non-negative) real number. Then the following holds:
$$ \mathbb{E}[ \mathrm{e}^{ \lambda \sum_{i=1}^{m}X_{i} } ] \leq \prod_{i=1}^{m} \mathbb{E}[\mathrm{e}^{\lambda X_{i} } ].$$
\end{restatable}

Note that Lemma \ref{Lem_NA_ExpectationProductInequality} lead directly to the following corollary relating iid and orthogonal ensembles:

\begin{restatable}[exponentials of OMCs and MCs]{corollary}{FirstCorollary}
\label{main_cor}
Let $\mathbf{z} \in \mathbb{R}^{d}$
and assume that function $f:\mathbb{R} \rightarrow \mathbb{R}$ is from the class $\textbf{F1}$.
Take an isotropic distribution $\mathcal{D}$ on $\mathbb{R}^{d}$, an ensemble of independent samples $\omega^{\mathrm{iid}}_{1},...,\omega^{\mathrm{iid}}_{s}$ and an orthogonal ensemble $\omega^{\mathrm{ort}}_{1},...,\omega^{\mathrm{ort}}_{s}$ giving rise to base MC estimator $\widehat{F}^{\mathrm{iid}}_{f,\mathcal{D}}(\mathbf{z})$ of $\mathbb{E}_{\omega \sim \mathcal{D}}[f(\omega^{\top}\mathbf{z})]$ and to its orthogonal version $\widehat{F}^{\mathrm{ort}}_{f,\mathcal{D}}(\mathbf{z})$. Then the following is true for any $\lambda$:
\begin{equation}
\mathbb{E}[e^{\lambda \widehat{F}^{\mathrm{ort}}_{f,\mathcal{D}}(\mathbf{z})}] \leq \mathbb{E}[e^{\lambda \widehat{F}^{\mathrm{iid}}_{f,\mathcal{D}}(\mathbf{z})}].    
\end{equation}
\end{restatable}

Corollary \ref{main_cor} enables us to obtain stronger concentrations results for OMCs than for base MCs for the class $\mathbf{F1}$. By combining it with extended Markov's inequality, we obtain the following: 

\begin{restatable}[OMC-bounds surpassing MC-bounds for the $\textbf{F1}$-class]{theorem}{FirstTheorem}
\label{first_thm}
Denote by $\mathrm{MSE}$ a mean squared error of the estimator, by $s$ the number of MC samples used and let $X=f(\omega^{\top}\mathbf{z})$ for $\omega \sim \mathcal{D}$. 
Then under assumptions as in Corollary \ref{main_cor}, OMC leads to the unbiased estimator satisfying for $\epsilon>0$:
\begin{equation}
\mathbb{P}[|\widehat{F}^{\mathrm{ort}}_{f,\mathcal{D}}(\mathbf{z})-F_{f,\mathcal{D}}(\mathbf{z})| \geq \epsilon] \leq p(\epsilon),    
\end{equation}
where $p(\epsilon)$ defined as: $p(\epsilon) \overset{\mathrm{def}}{=} \exp(-s\mathcal{L}_{X}(F_{f,\mathcal{D}}(\mathbf{z})+\epsilon)+\exp(-s\mathcal{L}_{X}(F_{f,\mathcal{D}}(\mathbf{z})-\epsilon)$
for unbounded $f$ and $p(\epsilon) \overset{\mathrm{def}}{=} 2\exp(-\frac{2 s \epsilon^2}{{(b-a)}^2})$ for $f \in [a,b]$, is a standard upper bound on 
$\mathbb{P}[|\widehat{F}^{\mathrm{iid}}_{f,\mathcal{D}}(\mathbf{z})-F_{f,\mathcal{D}}(\mathbf{z})| \geq \epsilon]$. Furthermore:
$
\mathrm{MSE}(\widehat{F}^{\mathrm{ort}}_{f,\mathcal{D}}(\mathbf{z})) \leq \mathrm{MSE}(\widehat{F}^{\mathrm{iid}}_{f,\mathcal{D}}(\mathbf{z}))    
$.
\end{restatable}

For functions from $\textbf{F2}$-class, we simply decompose $f$ into its monotone increasing ($f^{+}$) and decreasing ($f^{-}$) part, apply introduced tools independently to $f^{+}$ and $f^{-}$ and use union bound. Finally, if $f$ is taken from the $\textbf{F3}$-class, 
we first decompose it into $\mathrm{even}[f]$ and $\mathrm{odd}[f]$ components, by leaving only even/odd terms in the Taylor series expansion. We then observe that for isotropic distributions we have:  $F_{\mathrm{odd}[f],\mathcal{D}} = 0$ (see: Appendix Sec. \ref{sec:theorem1proof}),
and thus reduce the analysis to that of $\mathrm{even}[f]$ which is from the $\textbf{F2}$-class. We conclude that:

\begin{restatable}[Exponential bounds for OMCs and $\textbf{F2}/\textbf{F3}$ classes] {theorem} {SecondTheorem} \label{Thm_StrongConcentration_main}
Let $\mathbf{z} \in \mathbb{R}^{d}$ and assume that function $f:\mathbb{R} \rightarrow \mathbb{R}$ is from the class \textbf{F2} or \textbf{F3}. Then for $\epsilon>0$:
\vspace{-1.2mm}
\begin{equation}
\mathbb{P}[|\widehat{F}^{\mathrm{ort}}_{f,\mathcal{D}}(\mathbf{z})-F_{f,\mathcal{D}}(\mathbf{z})| \geq \epsilon] \leq
p(\epsilon) \overset{\mathrm{def}}{=} u^{+} + u^{-},
\end{equation}
where $u^{+/-} \overset{\mathrm{def}}{=} \exp(-s\mathcal{L}_{X^{+/-}}(F_{f,\mathcal{D}}(\mathbf{z})+\frac{\epsilon}{2})+\exp(-s\mathcal{L}_{X^{+/-}}(F_{f,\mathcal{D}}(\mathbf{z})-\frac{\epsilon}{2})$, and $X^{+/-}$ is defined as: $X^{+/-} \overset{\mathrm{def}}{=} f^{+/-}$ if $f$ is from $\mathbf{F2}$ and as: $X^{+/-} \overset{\mathrm{def}}{=} (\mathrm{even}[f])^{+/-}$ if $f$ is from $\mathbf{F3}$. As before, in the bounded case we can simplify $u^{+}$ and $u^{-}$ to: $u^{+/-} \overset{\mathrm{def}}{=} 2\exp(-\frac{s \epsilon^2}{2 {(b^{+/-}-a^{+/-})}^2})$,
where $a^{+},b^{+},a^{-},b^{-}$ are such that:
$f^{+} \in [a^{+},b^{+}]$ and $f^{-} \in [a^{-},b^{-}]$ if $f$ is from $\textbf{F2}$
or $(\mathrm{even}[f])^{+} \in [a^{+},b^{+}]$ and $(\mathrm{even}[f])^{-} \in [a^{-},b^{-}]$ if $f$ is from $\textbf{F3}$. Furthermore, if $(\mathrm{even}[f])^{+}=0$ or $(\mathrm{even}[f])^{-}=0$, we can tighten that bound using upper bound from Theorem \ref{first_thm} and thus, establish better concentration bounds than for base MC.

\end{restatable}


The proofs of all our theoretical results are given in the Appendix.
\vspace{-2.5mm}
\subsection{Applications}
\label{sec:applications}
\vspace{-2.5mm}
In this section we discuss in more detail applications of the presented results. We see that by taking $f(x)=x^{2}$, we can apply our results to the JLT setting from Sec. \ref{sec:tech_intro}. 
\vspace{-3.5mm}
\paragraph{General RBF kernels:} Even more interestingly, Theorem \ref{Thm_StrongConcentration_main} enables us to give first strong concentration results for all RBF kernels, avoiding very cumbersome technical requirements regarding tails of the corresponding spectral distributions (see: \cite{geom, psrnn}). In particular, we affirmatively answer an open question whether OMCs provide exponential concentration guarantees for the class of Mat\'{e}rn kernels for every value of the hyperparameter $\nu$ (\cite{geom}). Theorem \ref{Thm_StrongConcentration_main} can be also directly applied to obtain strong concentration results regarding kernel ridge regression with OMCs (see: Theorem 2 from \cite{psrnn}) for all RBF kernels as opposed to just \textit{smooth RBFs} as in \cite{psrnn}. Thus we bridge the gap between theory (previously valid mainly for Gaussian kernels) and practice (where improvements via OMCs were reported for RBF kernels across the board \cite{geom}).
\vspace{-3.5mm}
\paragraph{First Strong Results for Classes of PNG Kernels:} We also obtain first exponentially small upper bounds on errors for OMCs applied to PNG kernels, which were previously intractable and for which the best known results were coming from second moment methods \cite{unreas, kama}. To see this, note that PNGs defined by nonlinearity $h(x)=e^{cx}$ can be rewritten as functions from the class \textbf{F3} (with $\mathbf{z}=\mathbf{x}+\mathbf{y}$),
namely: $K_{h}(\mathbf{x},\mathbf{y})=\mathbb{E}_{\omega \sim \mathcal{N}(0,\mathbf{I}_{d})}[e^{c\omega^{\top}(\mathbf{x}+\mathbf{y})}]$ (see: Table 1).
Furthermore, by applying Theorem \ref{Thm_StrongConcentration_main}, we actually show that these bounds are better than for the base MC estimator.
\vspace{-2.5mm}
\subsubsection{Uniform Convergence for OMCs and New Kernel Ridge Regression Results}
\label{sec:uniform_convergence}
\vspace{-2.0mm}
Undoubtedly, one of the most important applications of results from Sec. \ref{sec:theory} are first uniform convergence guarantees for OMCs which provide a gateway to strong downstream guarantees for them, as we will show on the example of kernel ridge regression. MSE-results for OMCs from previous works suffice to provide some pointwise convergence, but are too weak for the uniform convergence and thus previous downstream theoretical guarantees for OMCs were not practical. The following is true and implies in particular that OMCs uniformly convergence for all RBF kernels :
\begin{restatable}[Uniform convergence for OMCs] {theorem} {ThirdTheorem}
\label{uniformtheory}
Let $\mathcal{M} \subseteq \mathbb{R}^{d}$ be compact with diameter $\mathrm{diam}(\mathcal{M})$. Assume that $f$ has Lipschitz constant $L_{f}$. Then under assumptions as in Theorem \ref{first_thm} / \ref{Thm_StrongConcentration_main}, for any $r>0$:
\begin{equation}
\mathbb{P}[\sup_{\mathbf{z} \in \mathcal{M}} 
|\widehat{F}^{\mathrm{ort}}_{f,\mathcal{D}}(\mathbf{z})-F_{f,\mathcal{D}}(\mathbf{z})| \geq \epsilon] \leq
C(\frac{\mathrm{diam}(\mathcal{M})}{r})^{d} \cdot p(\epsilon / 2) + (\frac{2r \sigma L_{f}}{\epsilon})^{2},
\end{equation}
where $\sigma^{2} = \mathbb{E}_{\omega \sim D}[\omega^{T}\omega]$ (i.e. the second moment of $D$), $p$ is as in RHS of inequality from Theorem \ref{first_thm} / \ref{Thm_StrongConcentration_main}
and $C>0$ is a universal constant. In particular, if boundedness conditions from Theorem \ref{first_thm} / \ref{Thm_StrongConcentration_main} are satisfied, one can take: $s=\Theta(\frac{d}{\epsilon^{2}}\log(\frac{\sigma L_{f} \mathrm{diam}(\mathcal{M})}{\epsilon}))$ to get uniform $\epsilon$-error approximation. Moreover, even if boundedness conditions are not satisfied, one can still take: $s = \Theta(d\log(\frac{L_{f}\sigma(\mathrm{diam}(\mathcal{M}))}{\epsilon}))$to get uniform $\epsilon$-error approximation.
\end{restatable}
\vspace{-2mm}
We can directly apply these results to kernel ridge regression with an \textbf{arbitrary} RBF via OMCs, by noting that in the RHS of Theorem 2 from \cite{psrnn} upper-bounding the error, we can drop $N^{2}$ multiplicative factor (if all points are in the compact set) (Appendix: Sec.\ref{ker}). This term was added as a consequence of simple union bound, no longer necessary if uniform convergence is satisfied.

\vspace{-3.0mm}
\section{Near-Orthogonal Monte Carlo Algorithm}
\vspace{-3mm}
Near-Orthogonal Monte Carlo (or: NOMC) is a new paradigm for constructing entangled MC samples for estimators involving isotropic distributions if the number of samples required satisfies $s>d$. 
We construct NOMC-samples to make angles between any two samples close to orthogonal (note that they cannot be exactly orthogonal for $s>d$ since this would imply their linear independence).
That makes their distribution much more uniform than in other methods (see: Fig. \ref{visualization}).
\small
\begin{figure}[h]
\vspace{-3mm}
    \begin{minipage}{1.0\textwidth}
    \includegraphics[width=.99\linewidth]{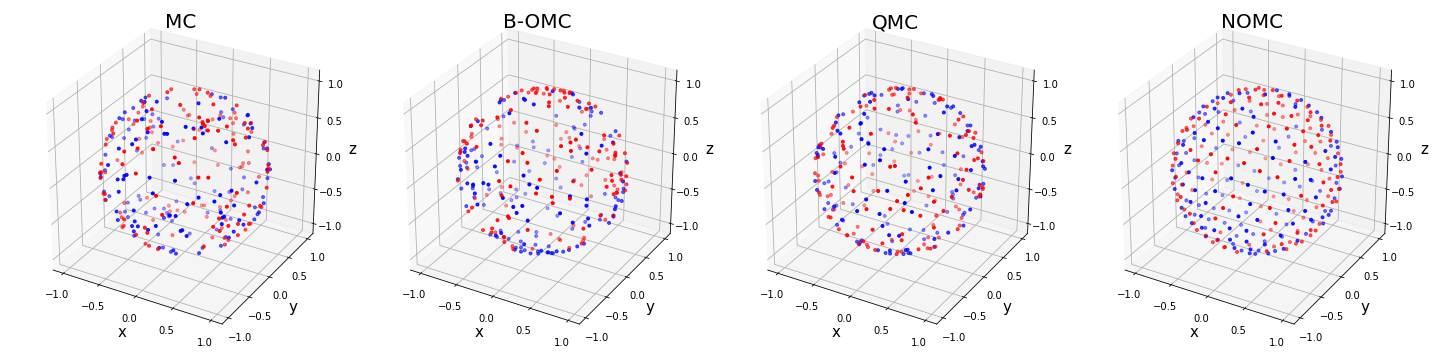}
    \end{minipage}
    \vspace{-4.5mm}
    \caption{\small{Visual comparison of the distribution of samples for four methods for $d=3$ and $s=150$. From left to right: base MC, B-OMC, QMC using Halton sequences and our NOMC. We see that NOMC produces most uniformly distributed samples.}}
\label{visualization}
\vspace{-3mm}
\end{figure}
\normalsize

This has crucial positive impact on the accuracy of the estimators applying NOMCs, making them superior to other methods, as we demonstrate in Sec. \ref{sec:experiments}, and still unbiased.

There are two main strategies that we apply to obtain near-orthogonality surpassing this in QMC or B-OMC. Our first proposition is to cast sample-construction as an optimization problem, where near-orthogonal entanglement is a result of optimizing objectives involving angles between samples. We call this approach: $\mathrm{opt}$-NOMC. Even though such an optimization incurs only one-time additional cost, we also present $\mathrm{alg}$-NOMC algorithm that has lower time complexity and is based on the theory of algebraic varieties over finite fields. Algorithm $\mathrm{alg}$-NOMC does not require optimization and in practice gives similar accuracy, thus in the experimental section we refer to both simply as NOMC. Below we discuss both strategies in more detail.
\vspace{-3mm}
\subsection{Algorithm $\mathrm{opt}$-NOMC}
\vspace{-2.5mm}
The idea of Algorithm $\mathrm{opt}$-NOMC is to force repelling property of the samples/particles (that for the one-block case was guaranteed by the ND-property) via specially designed energy function. 

That energy function achieves lower values for well-spread samples/particles and can be rewritten as the sum over energies $E(\omega_{i},\omega_{j})$ of local particle-particle interactions. There are many good candidates for $E(\omega_{i},\omega_{j})$. We chose: 
$E(\omega_{i},\omega_{j}) = \frac{\delta}{\delta + \|\omega_{i}-\omega_{j}\|^{2}_{2}}$,    
where $\delta>0$ is a tunable hyperparameter.
We minimize such an energy function on the sphere using standard gradient descent approach with projections.
WLOG we can assume that the isotropic distribution $\mathcal{D}$ under consideration is a uniform distribution on the sphere $\mathrm{Unif}(\mathcal{S}^{d-1})$, since for other isotropic distributions we will only need to conduct later cheap renormalization of samples' lengths. When the optimization is completed, we return randomly rotated ensemble, where random rotation is encoded by Gaussian orthogonal matrix obtained via standard Gram-Schmidt orthogonalization of the Gaussian unstructured matrix (see: \cite{ort}). Random rotations enable us to obtain correct marginal distributions (while keeping the entanglement of different samples obtained via optimization) and consequently - unbiased estimators when such ensembles are applied.
We initialize optimization with an ensemble corresponding to B-OMC as a good quality starting point.
For the pseudocode of $\mathrm{opt}$-NOMC, see $\mathrm{Algorithm}$ 1 box.

\textbf{Remark:} Note that even though in higher-dimensional settings, such an optimization is more expensive, this is a $\textbf{one time cost}$. If new random ensemble is needed, it suffices to apply new random rotation on the already optimized ensemble. Applying such a random rotation is much cheaper and can be further sped up by using proxies of random rotations (see: \cite{unreas}). For further discussion regarding the cost of the optimization (see: Appendix: Sec \ref{appendix:clocktime}).

\begin{algorithm}[H]\label{algorithm_NOMC}\
\textbf{Input}: Parameter $\delta, \eta, T$ \;
\textbf{Output}: randomly rotated ensemble $ \mathbf{\omega_{i}}^{(T)}$ for $i=1,2,...,N$ \;

Initialize $ \mathbf{\omega_{i}^{(0)}} (i=1,2,...,N)$ with multiple orthogonal blocks in B-OMC

 \For{$t=0,1,2,...,T-1 $}
 {
    Calculate Energy Function $E(\mathbf{\omega}_{i}^{(t)},\mathbf{\omega}_{j}^{(t)})=
    \frac{\delta}{\delta+\|\mathbf{\omega}_{i}^{(t)}-\mathbf{\omega}_{j}^{(t)}\|_{2}^{2}}$ for $i \neq j \in \{1,...,N\} $ \;
    
    \For{$i=1,2,...,N $}
    {
        Compute gradients $F_{i}^{(t)}=\frac{\partial \sum_{i\neq j} E(\mathbf{\omega_{i}}^{(t)},\mathbf{\omega_{j}}^{(t)})}
        {\partial \mathbf{\omega_{i}}^{(t)}}$ \;
    
        Update  $ \mathbf{\omega_{i}}^{(t+1)} \leftarrow  \mathbf{\omega_{i}}^{(t)} - \eta F_{i}^{(t)}$\;
        
        Normalize $ \mathbf{\omega_{i}}^{(t+1)}  \leftarrow  \frac{\mathbf{\omega_{i}}^{(t+1)}}{\|\mathbf{\omega_{i}}^{(t+1)}\|_{2}} $  \;        
    }
 }
 \caption{Near Orthogonal Monte Carlo: $\mathrm{opt}$-NOMC variant}
\end{algorithm}
\vspace{-3mm}
\subsection{Algorithm $\mathrm{alg}$-NOMC}
As above, without loss of generality we will assume here that $\mathcal{D} = \mathrm{Unif}(\mathcal{S}^{d-1})$ since, as we mentioned above, we can obtain samples for general isotropic $\mathcal{D}$ from the one for $\mathrm{Unif}(\mathcal{S}^{d-1})$ by simple length renormalization.
Note that in that setting we can quantify how well the samples from the ensemble $\Omega$ are spread by computing $\mathcal{A}(\Omega) \overset{\mathrm{def}}{=} \max_{i} |\omega_{i}^{\top}\omega_{j}|$.
It is a standard fact from probability theory that for base MC samples $\mathcal{A}(\Omega)=\Theta(r^{\frac{1}{2}}d^{-\frac{1}{2}}\sqrt{\log(d)})$ with high probability if the size of $\Omega$ satisfies: $|\Omega|=d^{r}$ and that is the case also for B-OMC. The question arises: can we do better ?

It turns out that the answer is provided by the theory of algebraic varieties over finite fields. Without loss of generality, we will assume that $d=2p$, where $p$ is prime.
We will encode samples from our structured ensembles via complex-valued functions $g_{c_{1},...,c_{r}}:\mathbb{F}_{p} \rightarrow \mathbb{C}$, given as
\begin{equation}
g_{c_{1},...,c_{r}}(x)=\frac{1}{\sqrt{p}}\exp(\frac{2 \pi i(c_{r}x^{r}+...+c_{1}x)}{p}),
\end{equation} where $\mathbb{F}_{p}$ and $\mathbb{C}$ stand for the field of residues modulo $p$ and a field of complex numbers respectively and $c_{1},...,c_{r} \in \mathbb{F}_{p}$.
The encoding $\mathbb{C}^{\mathbb{F}_{p}} \rightarrow \mathbb{R}^{d}$ is as follows: 
\begin{equation}
g_{c_{1},...,c_{r}}(x) \rightarrow \mathbf{v}(c_{1},...,c_{r})\overset{\mathrm{def}}{=}(a_{1},b_{1},...,a_{p},b_{p})^{\top} \in \mathbb{R}^{d},
\end{equation}
where:  $g_{c_{1},...,c_{r}}(j-1)=a_{j} + ib_{j}$. Using Weil conjecture for curves, one can show \cite{weil} that:
\begin{lemma}[NOMC via algebraic varieties]
If $\Omega=\{\mathbf{v}(c_{1},...,c_{r})\}_{c_{1},...,c_{r} \in \mathbb{F}(p)} \in S^{d-1}$, then $|\Omega|=p^{r}$, and furthermore
$\mathcal{A}(\Omega) \leq (r-1)p^{-\frac{1}{2}}$.
\end{lemma}
\vspace{-2mm}
Thus we see that we managed to get rid of the $\sqrt{\log(d)}$ factor as compared to base MC samples and consequently, obtain better quality ensemble.
As for $\mathrm{opt}$-NOMC, before returning samples, we apply random rotation to the entire ensemble. But in contrary to $\mathrm{opt}$-NOMC, in this construction we avoid any optimization, and any more expensive (even one time) computations.
\vspace{-3mm}
\section{Experiments}
\label{sec:experiments}
\vspace{-2mm}
We empirically tested NOMCs in two settings: \textbf{(1)} kernel approximation via random feature maps and \textbf{(2)} estimating sliced Wasserstein distances, routinely used in generative modeling \cite{acharya}. For (\textbf{1}), we tested the effectiveness of NOMCs for RBF kernels, non-RBF shift-invariant kernels as well as several PNG kernels. For (\textbf{2}), we considered different classes of multivariate distributions. 
As we have explained in Sec. \ref{sec:tech_intro}, the sliced Wasserstein distance for two distributions $\eta, \mu$ is given as:
\begin{equation} \label{equation:swd}
\mathrm{SWD}(\eta, \mu) = (\mathbb{E}_{\mathbf{u} \sim \mathrm{Unif}(\mathcal{S}^{d-1})}[\mathrm{WD}^{p}_{p}(\eta_{\mathbf{u}},\mu_{\mathbf{u}})])^{\frac{1}{p}}.   
\end{equation}
In our experiment we took $p=2$.
We compared against three other methods: \textbf{(a)} base Monte Carlo (MC), \textbf{(b)} Quasi Monte Carlo applying Halton sequences (QMC)(\cite{qmc_1}) and 
block orthogonal MC (B-OMC). Additional experimental details are in the Appendix (Sec. \ref{appendix:swd}).
The results are presented in Fig. \ref{fig:kernels} and Fig. \ref{fig:distributions}. 
Empirical MSEs were computed by averaging over $k=450$ independent experiments.
Our NOMC method clearly outperforms other algorithms. For kernel approximation NOMC provides the best accuracy for $7$ out of $8$ different classes of kernels and for the remaining one it is the second best. For SWD approximation, NOMC provides the best accuracy \textbf{for all $8$ classes} of tested distributions. To the best of our knowledge, NOMC is the first method outperforming B-OMC.
\vspace{-3.5mm}
\small
\begin{figure}[h]
    \begin{minipage}{1.0\textwidth}
    \includegraphics[width=.99\linewidth]{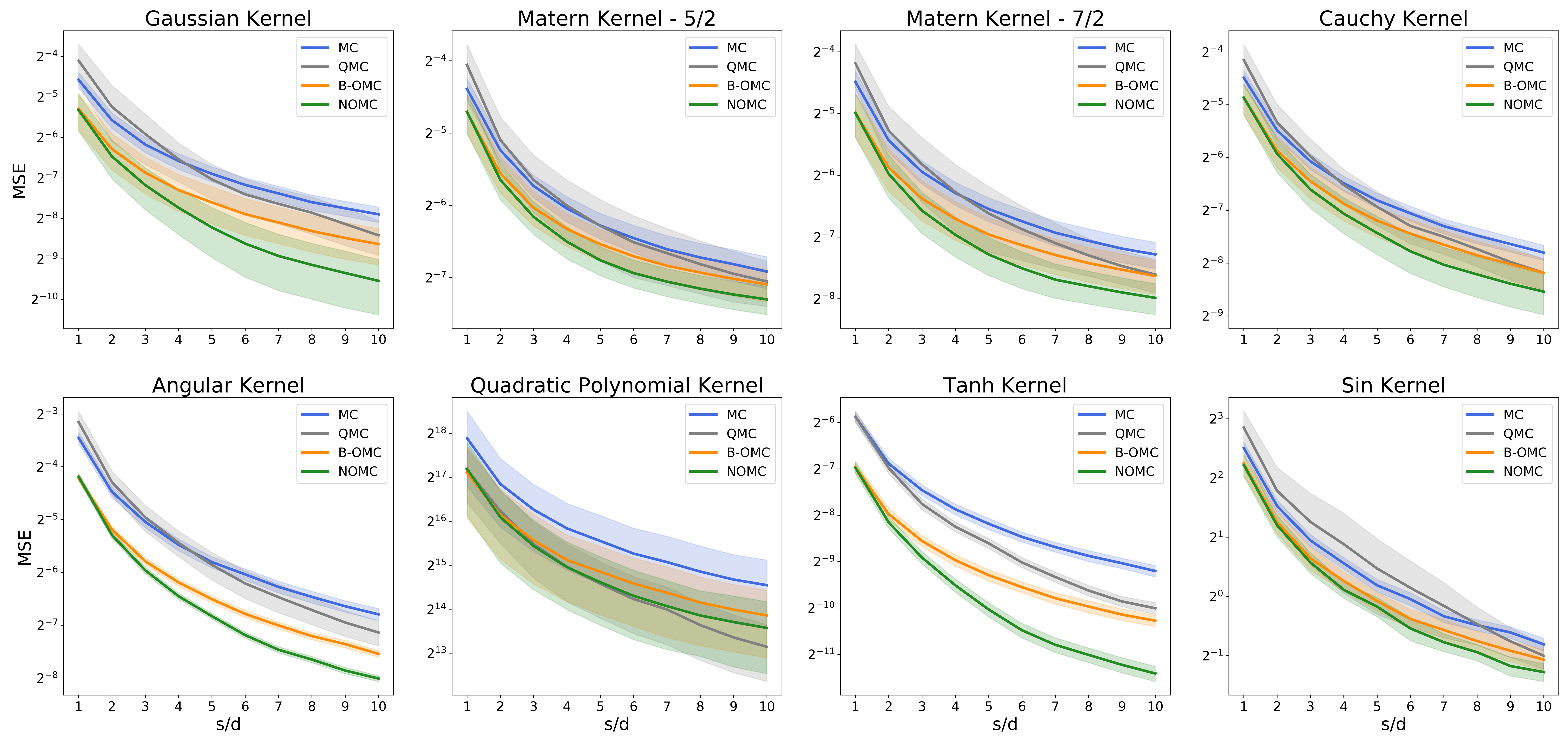}
    \end{minipage}
    \vspace{-4.5mm}
    \caption{\small{Comparison of MSEs of estimators using different sampling methods: MC, QMC, B-OMC and our NOMC. First four tested kernels are shift-invariant (first three are even RBFs) and last four are PNGs with name indicating nonlinear mapping $h$ used (see: Sec. \ref{sec:tech_intro}). On the x-axis: the number of blocks (i.e. the ratio of the number of samples $D$ used and data dimensionality $d$). Shaded region corresponds to $0.5 \times \mathrm{std}$}.}
\label{fig:kernels}
\end{figure}
\normalsize

\small
\begin{figure}[h]
    \begin{minipage}{1.0\textwidth}
    \includegraphics[width=.99\linewidth]{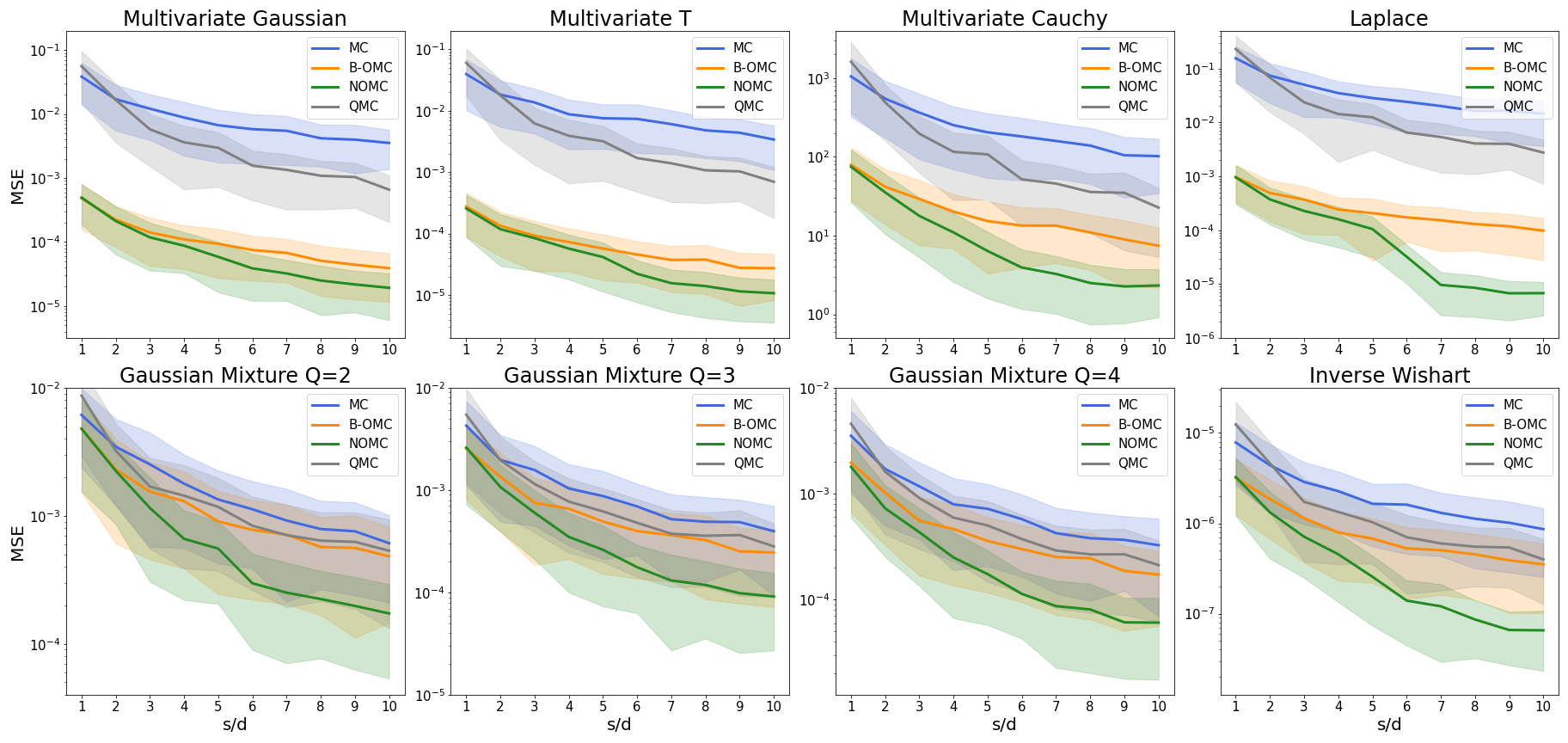}
    \end{minipage}
    \vspace{-4.5mm}
    \caption{\small{As in Fig. \ref{fig:kernels}}, but this time we compare estimators of sliced Wasserstein distances (SWDs) between two distributions taken from a class which name is given above the plot.}
\label{fig:distributions}
\end{figure}
\normalsize
\section{Broader Impact}
\label{sec:broader_impact}

In this paper we presented first general theory for the prominent class of orthogonal Monte Carlo (OMC) estimators (used on a regular basis for variance reduction), by discovering an intriguing connection with the theory of negatively dependent random variables. In particular, we give first results for general nonlinear mappings and for \textbf{all} RBF kernels as well as first uniform convergence guarantees for OMCs. Inspired by developed theory, we also propose new Monte Carlo algorithm based on near-orthogonal samples (NOMC) that outperforms previous SOTA in the notorious setting, where number of required samples exceeds data dimensionality.

We do believe that those findings have several important consequences for theoreticians as well as practitioners working on Monte Carlo methods for machine learning:

\textbf{General Nonlinear Models:} Understanding the impact of structured Monte Carlo methods leveraging entangled ensembles for general nonlinear models is of crucial importance in machine learning and should guide the research on the developments of new more sample-efficient and accurate MC methods. We think about our results as a first step towards this goal.

\textbf{Uniform Convergence Results:} Our uniform convergence results for OMCs from Section \ref{sec:uniform_convergence} are the first such guarantees for OMC methods that can be applied to obtain strong downstream guarantees for OMCs. We demonstrated it on the example of kernel ridge regression, but similar results can be derived for other downstream applications such as kernel-$\mathrm{SVM}$. They are important since in particular they provide detailed guidance on how to choose in practice the number of random features (see: the asymptotic formula for the number of samples in Theorem \ref{uniformtheory}). 

\textbf{Evolutionary Strategies with Structured MC:} We showed the value of our NOMC algorithm in Sec. \ref{sec:experiments} for kernel and SWD approximation, but the method can be applied as a general tool in several downstream applications, where MC sampling from isotropic distributions is required, in particular in evolutionary strategies (ES) for training reinforcement learning policies \cite{es_ort}. 
ES techniques became recently increasingly popular as providing state-of-the-art algorithms for tasks of critical importance in robotics such as end-to-end training of high-frequency controllers \cite{tennis} as well as training adaptable meta-policies \cite{meta-policy}. ES methods heavily rely on Monte Carlo estimators of gradients of Gaussians smoothings of certain classes of functions. This makes them potential beneficiaries of new developments in the theory of Monte Carlo sampling and consequently, new Monte Carlo algorithms such as NOMC.

\textbf{Algebraic Monte Carlo:} We also think that proposed by us NOMC algorithm in its algebraic variant is one of a very few effective ways of incorporating deep algebraic results into the practice of MC in machine learning. Several QMC methods rely on number theory constructions, but, as we presented, these are much less accurate and in practice not competitive with other structured methods. Not only does our $\mathrm{alg}$-NOMC provide strong theoretical foundations, but it gives additional substantial accuracy gains on the top of already well-optimized methods with no additional computational cost.
This motivates future work on incorporating modern algebraic techniques into Monte Carlo algorithms for machine learning.

\bibliographystyle{abbrv}
\bibliography{nomc}

\newpage
\section*{APPENDIX A: Demystifying Orthogonal Monte Carlo and Beyond - Proofs of Theoretical Results}

For the convenience of Reader, here we restate the theorems first and then present their proofs.

\subsection{Proof of Lemma \ref{nd-lemma}}

\begin{proof}
From the definition of negative dependence, what we need to prove is:
\begin{equation}
\mathbb{P}[ \bigcap_i^{d} (|{w_i^\mathrm{ort}}^{\top} \mathbf{z}| \leq \tilde{x}_i) ] \leq \prod_i^{d} \mathbb{P}[|{w_i^\mathrm{ort}}^{\top} \mathbf{z}| \leq \tilde{x}_i]    
\end{equation}
\begin{equation}
\mathbb{P}[ \bigcap_i^{d} (|{w_i^\mathrm{ort}}^{\top} \mathbf{z}| \geq \tilde{x}_i) ] \leq \prod_i^{d} \mathbb{P}[|{w_i^\mathrm{ort}}^{\top} \mathbf{z}| \geq \tilde{x}_i]
\end{equation}
where we use $ \tilde{x}_i$  to represent a different value than the original $x_i$, which should be $f^{-1}(x_i)$. We will illustrate how to prove the first inequality here since the other can be proved accordingly.

Firstly, we can decompose $w_i^\mathrm{ort}=v_i^\mathrm{ort} l_i$, where $v_i^\mathrm{ort}$ has unit length, and $l_i$ is taken independently from $v_i^\mathrm{ort}$, which represents the length scalar.
So we need to prove the following:
\begin{equation}
\mathbb{P}[ \bigcap_i^{d} (|{v_i^\mathrm{ort}}^{\top} \mathbf{z}| \leq \frac{\tilde{x}_i}{l_i} ) ] \leq \prod_i^{d} \mathbb{P}[|{v_i^\mathrm{ort}}^{\top} \mathbf{z}| \leq \frac{\tilde{x}_i}{l_i}]   
\end{equation}

But actually since negatively dependence should holds for any $x_i \in \mathbb{R}$, so it actually does not matters which scalar we use in the right hand side of each part of the probability inequality. So we will continue to use $x_i$ instead of $\frac{\tilde{x}_i}{l_i}$ in the following proof.

Furthermore, we assume $\| \mathbf{z}\|_2=1$ without loss of generality. Proof for the cases when $x_i\geq 1$ or $x_i\leq 0$ is trivial under such assumption, so we will only concentrate on the case when $0< x_j < 1$. Here, we can use a second trick for distribution transformation. We regard $v_1^\mathrm{ort}$, $v_2^\mathrm{ort}$, ..., $v_d^\mathrm{ort}$ as fixed, and $\mathbf{z}$ as a random rotation vector, so that we can replace $v_1^\mathrm{ort}$, $v_2^\mathrm{ort}$, ..., $v_d^\mathrm{ort}$ as $e_1, e_2, ..., e_d$ and $\mathbf{z}$ be a unit length vector uniformly distributed on the $\mathcal{S}^{d-1}$. After such transformation, the distribution of $|{v_i^\mathrm{ort}}^{\top} \mathbf{z}|$ will be equivalent to $ \frac{|e_i^{\top} g |}{{\| g \|}_2}=\frac{g_i}{{\| g \|}_2}$, where $g$ is a gaussian vector, and $g_i$ is its length of projection onto the $i^\mathrm{th}$ coordinate.

So the problem we need to prove is transformed to the following inequality:
\begin{equation}
\mathbb{P}[ \bigcap_i^{d} (\frac{g_i}{{\| g \|}_2}\leq x_i) ] \leq \prod_i^{d} \mathbb{P}[\frac{g_i}{{\| g \|}_2}\leq x_i] 
\end{equation}
From the rule of conditional probability, the LHS can be transformed to: 
\begin{equation}
\mathbb{P}[\frac{g_1}{{\| g \|}_2}\leq x_1] \mathbb{P}[\frac{g_2}{{\| g \|}_2}\leq x_2 | \frac{g_1}{{\| g \|}_2}\leq x_1]\mathbb{P}[\frac{g_3}{{\| g \|}_2}\leq x_3 | (\frac{g_1}{{\| g \|}_2}\leq x_1)\cap (\frac{g_2}{{\| g \|}_2}\leq x_2) ]...    
\end{equation}
until the conditional probability of $\frac{g_d}{{\| g \|}_2}$ on all $\frac{g_i}{{\| g \|}_2}$ for $i=1,2,...,d-1$ .

Therefore, we conclude that we only need to prove the following for each corresponding term $i$: 
\begin{equation}
\mathbb{P}[\frac{g_i}{{\| g \|}_2}\leq x_i \vert \bigcap_{j=1}^{i-1} (\frac{g_j}{{\Vert g \Vert}_2}\leq x_j) ] \leq P(\frac{g_i}{{\| g \|}_2}\leq x_i)
\end{equation}

We note that $\frac{g_j}{{\Vert g \Vert}_2}\leq x_j$ is equivalent to $\frac{{g_j}^2}{{\Vert g \Vert}_{2}^{2}} \leq {x_j}^2$. So for each $j<i$ we have: 
\begin{equation}
{g_j}^2 \leq {x_j}^2 {g_i}^2 + {x_j}^2 ({g_1}^2 + ... + {g_{i-1}}^2 + {g_{i+1}}^2 +...+{g_{d}}^2  )    
\end{equation}
which can be rewrite as:
\begin{equation}
{g_i}^2 \geq \frac{{g_j}^2}{{x_j}^2} - ({g_1}^2 + ... + {g_{i-1}}^2 + {g_{i+1}}^2 +...+{g_{d}}^2  )  
\end{equation}
We can also rewrite from $\frac{g_i}{{\Vert g \Vert}_2}\leq x_i$ and derive:
\begin{equation}
{g_i}^2 \leq \frac{{x_i}^2}{(1-{x_i}^2)} ({g_1}^2 + ... + {g_{i-1}}^2 + {g_{i+1}}^2 +...+{g_{d}}^2  )  
\end{equation}

Therefore, 
\begin{equation} \label{equation:last_in_ND_lemma}
\begin{split}
\mathbb{P}[\frac{g_i}{{\| g \|}_2}\leq x_i |\bigcap_{j=1}^{i-1} (\frac{g_j}{{\| g \|}_2}\leq x_j) ] = \mathbb{P}[{g_i}^2 \leq \frac{{x_i}^2}{(1-{x_i}^2)} ({g_1}^2 + ... + {g_{i-1}}^2 + {g_{i+1}}^2 +...+{g_{d}}^2  ) \\ 
| \bigcap_{j=1}^{i-1} ( {g_i}^2 \geq \frac{{g_j}^2}{{x_j}^2} - ({g_1}^2 + ... + {g_{i-1}}^2 + {g_{i+1}}^2 +...+{g_{d}}^2  ) )  ] \\
\leq  \mathbb{P}[{g_i}^2 \leq \frac{{x_i}^2}{(1-{x_i}^2)} ({g_1}^2 + ... + {g_{i-1}}^2 + {g_{i+1}}^2 +...+{g_{d}}^2  )] = \mathbb{P}[\frac{g_i}{{\| g \|}_2 }\leq x_i]    
\end{split}
\end{equation}
which finishes our proof of negative dependence. 

\end{proof}

We can conclude that:

\FirstLemma*

\subsection{Proof of Lemma 2}

To prove Lemma 2, we will use the following result \cite{Rita208}, \cite{Dev983}:

\begin{lemma}  \label{Lem_ND_2points} 
Let $X_1,…,X_n$ be negatively dependent random variables, then:
\begin{itemize}
  \item If $f_1,…,f_n$ is a sequence of measurable functions which are all monotone non-decreasing (or all are monotone non-increasing), then $f_1(X_1),…,f_n(X_n)$ are also negatively dependent random variables.
  \item $\mathbb{E}[X_1…X_n ] \leq \mathbb{E}[X_1]…\mathbb{E}[X_n]$, provided the expectation exist.
\end{itemize}
\end{lemma}


\SecondLemma*

\begin{proof} 
By Lemma 1 and the first point iof Lemma \ref{Lem_ND_2points}, we know that $X_{1},...,X_{n}$ are negatively dependent. Then, by applying the second point in Lemma \ref{Lem_ND_2points}, we know that:
\begin{equation}
\mathbb{E}[f_1(X_1 )…f_n(X_n)] \leq \mathbb{E}[f_1(X_1)]…\mathbb{E}[f_n(X_n)]
\end{equation}

If $ \lambda \geq 0$, we can take a non-decreasing function $ f_i(X_i) = e^{\lambda X_i}$ for each $i$, then: 
\begin{equation}
\mathbb{E}[\exp(\lambda \sum_{i=1}^m X_i)] \leq \prod_{i=1}^m \mathbb{E}[e^{\lambda X_i}] 
\end{equation}
Similarly, if $ \lambda \leq 0$, then we can take a non-increasing function $f_i(X_i) = e^{\lambda X_i}$, and this inequality will also be true. Actually, we say that $X_1,...X_n$ are acceptable if the inequality $\mathbb{E}[\exp(\lambda \sum_{i=1}^m X_i)] \leq \prod_{i=1}^m \mathbb{E}[e^{\lambda X_i}]$ holds for any real $\lambda$ \cite{Rita208}. 
\end{proof}

\subsection{Proof of Corollary \ref{main_cor}}

\FirstCorollary*

\begin{proof}

From Lemma \ref{nd-lemma} and Lemma \ref{Lem_NA_ExpectationProductInequality}, we can derive directly that if the function $f$ is monotone increasing (or decreasing) in $| \omega_i^{\top} \mathbf{z} |$, and we define $\widehat{F}^{\mathrm{ort}}_{f,\mathcal{D}}(\mathbf{z})$ and $\widehat{F}^{\mathrm{iid}}_{f,\mathcal{D}}(\mathbf{z})$ as the orthogonal and iid estimates for $\mathbb{E}_{\omega \sim \mathcal{D}} [f_{\mathcal{Z}}(\omega)]$, then the ND of $|{\omega_1^{ort}}^{\top} \mathbf{z}|,...,|{\omega_d^{ort}}^{\top} \mathbf{z}|$ implies $\forall \lambda \in \mathbf{R}$ and $s=d$:
\begin{equation} \label{equation:cor1Second}
\mathbb{E}[\exp(\lambda \widehat{F}^{\mathrm{ort}}_{f,\mathcal{D}}(\mathbf{z}))] \leq \prod_{i=1}^d \mathbb{E}[e^{\lambda\widehat{F}^{\mathrm{ort}}_{f,\mathcal{D}}(\mathbf{z})}] = \prod_{i=1}^d \mathbb{E}[e^{\lambda \widehat{F}^{\mathrm{iid}}_{f,\mathcal{D}}(\mathbf{z})}] 
\end{equation}
which is exactly the inequality in this Corollary.

For $s=kd$ where $k$ is a multiplier larger than 1, we can define $\widehat{F}^{\mathrm{ort}}_{f,\mathcal{D}}(\mathbf{z})$ as the estimator constructed by stacking $k$ independent orthogonal blocks together with dimension $d$, and $\widehat{F}^{\mathrm{iid}}_{f,\mathcal{D}}(\mathbf{z})$ as the base estimator with $s$ samples. The proof in such case is trivial since we can decompose $\mathbb{E}[\exp(\lambda \widehat{F}^{\mathrm{ort}}_{f,\mathcal{D}}(\mathbf{z}))]$ into the multiplication of $k$ expectations of independent blocks, and then use equation (\ref{equation:cor1Second}) again.

\end{proof}

\subsection{Proof of Theorem \ref{first_thm}}
\label{sec:theorem1proof}
\FirstTheorem*

\begin{proof}

Let's first work on the case when function $f$ is bounded. In such case, we can apply Chernoff-Hoeffdings inequality for iid estimators to $p(\epsilon)$, which is $2\exp(-\frac{2 s \epsilon^2}{{(b-a)}^2})$.

For $\lambda>0,\epsilon \in\mathbb{R}$, we apply Markov inequality here:
\begin{equation}
\begin{split}
\mathbb{P}[\widehat{F}^{\mathrm{ort}}_{f,\mathcal{D}}(\mathbf{z})-F_{f,\mathcal{D}}(\mathbf{z}) \geq \epsilon] = 
\mathbb{P}[e^{\lambda(\widehat{F}^{\mathrm{ort}}_{f,\mathcal{D}}(\mathbf{z})-F_{f,\mathcal{D}}(\mathbf{z}))} \geq e^{\lambda \epsilon}]  \\
\leq e^{-\lambda \epsilon} \mathbb{E}[e^{\lambda(\widehat{F}^{\mathrm{ort}}_{f,\mathcal{D}}(\mathbf{z})-F_{f,\mathcal{D}}(\mathbf{z}))}] =
e^{-\lambda \epsilon}e^{-\lambda F_{f,\mathcal{D}}(\mathbf{z})} \mathbb{E}[e^{\lambda \widehat{F}^{\mathrm{ort}}_{f,\mathcal{D}}(\mathbf{z})}] 
\end{split}
\end{equation}
Similarly, for iid estimator, we have:
\begin{equation}
\mathbb{P}[\widehat{F}^{\mathrm{iid}}_{f,\mathcal{D}}(\mathbf{z})-F_{f,\mathcal{D}}(\mathbf{z}) \geq \epsilon] = e^{-\lambda \epsilon}e^{-\lambda F_{f,\mathcal{D}}(\mathbf{z})} \mathbb{E}[e^{\lambda \widehat{F}^{\mathrm{iid}}_{f,\mathcal{D}}(\mathbf{z})}] 
\end{equation}
From Corollary \ref{main_cor}, we know directly that orthogonal estimator has better upper bound than iid estimator. 

For $\lambda<0$, 
\begin{equation}
\mathbb{P}[\widehat{F}^{\mathrm{ort}}_{f,\mathcal{D}}(\mathbf{z})-F_{f,\mathcal{D}}(\mathbf{z}) \leq -\epsilon] = 
\mathbb{P}[e^{\lambda(\widehat{F}^{\mathrm{ort}}_{f,\mathcal{D}}(\mathbf{z})-F_{f,\mathcal{D}}(\mathbf{z}))} \geq e^{\lambda \epsilon}]    
\end{equation}

$\mathbb{E}[e^{\lambda \widehat{F}^{\mathrm{ort}}_{f,\mathcal{D}}(\mathbf{z})}] \leq \mathbb{E}[e^{\lambda \widehat{F}^{\mathrm{iid}}_{f,\mathcal{D}}(\mathbf{z})}]$ in Corollary \ref{main_cor} also guarantees better lower bound than iid estimator. 

By combining these two cases, we know that $\mathbb{P}[|\widehat{F}^{\mathrm{ort}}_{f,\mathcal{D}}(\mathbf{z})-F_{f,\mathcal{D}}(\mathbf{z})| \geq \epsilon]$ has better bound than $\mathbb{P}[|\widehat{F}^{\mathrm{iid}}_{f,\mathcal{D}}(\mathbf{z})-F_{f,\mathcal{D}}(\mathbf{z})| \geq \epsilon]$.

Then for unbounded function $f$, we can apply Cra$\mathrm{\Acute{m}}$er-Chernoff bound to $p(\epsilon)$. We rewrite several steps here to show: $$p(\epsilon)=\exp\{-s(\mathcal{L}_{X}(F_{f,\mathcal{D}}(\mathbf{z}))+\epsilon)\}+\exp\{-s(\mathcal{L}_{X}(F_{f,\mathcal{D}}(\mathbf{z}))-\epsilon)\}$$
which is the bound for iid estimator.
\begin{align}
    &\quad\ \mathbb{P}[|\widehat{F}^{\mathrm{iid}}_{f,\mathcal{D}}(\mathbf{z})-F_{f,\mathcal{D}}(\mathbf{z})| \geq \epsilon] \notag\\
    &= \mathbb{P}[\widehat{F}^{\mathrm{iid}}_{f,\mathcal{D}}(\mathbf{z})-F_{f,\mathcal{D}}(\mathbf{z}) \geq \epsilon] +
\mathbb{P}[\widehat{F}^{\mathrm{iid}}_{f,\mathcal{D}}(\mathbf{z})-F_{f,\mathcal{D}}(\mathbf{z}) \leq -\epsilon] \notag\\
&= \mathbb{P}[\sum_{i=1}^{s} f_{\mathcal{Z}}(\omega_{i}^{\mathrm{iid}})- s F_{f,\mathcal{D}}(\mathbf{z}) \geq s \epsilon] +
\mathbb{P}[\sum_{i=1}^{s} f_{\mathcal{Z}}(\omega_{i}^{\mathrm{iid}})- s F_{f,\mathcal{D}}(\mathbf{z}) \leq -s \epsilon] \notag\\
&\leq \exp\{-\sup_{\theta>0}(\theta s\epsilon - \mathrm{log}^{\mathbb{E}[e^{\theta(\sum_{i=1}^s f_{\mathcal{Z}}(\omega_{i}^{\mathrm{iid}}) - s F_{f,\mathcal{D}}(\mathbf{z}))}]})\} 
\notag\\ &\quad\quad\quad
+ \exp\{-\sup_{\theta<0}(-\theta s\epsilon - \mathrm{log}^{\mathbb{E}[e^{\theta(\sum_{i=1}^s f_{\mathcal{Z}}(\omega_{i}^{\mathrm{iid}}) - s F_{f,\mathcal{D}}(\mathbf{z}))}]})\} \notag\\
&= \exp\{-s\sup_{\theta>0}(\theta \epsilon - \frac{1}{s}\mathrm{log}^{\mathbb{E}[e^{\theta(\sum_{i=1}^s f_{\mathcal{Z}}(\omega_{i}^{\mathrm{iid}}))}]}+\theta F_{f,\mathcal{D}}(\mathbf{z}))\} \notag\\
&\quad\quad\quad+ \exp\{-s\sup_{\theta<0}(-\theta \epsilon - \frac{1}{s}\mathrm{log}^{\mathbb{E}[e^{\theta(\sum_{i=1}^s f_{\mathcal{Z}}(\omega_{i}^{\mathrm{iid}}))}]}+\theta F_{f,\mathcal{D}}(\mathbf{z}))\} \notag\\
&= \exp\{-s\sup_{\theta>0}( \theta \epsilon - \mathrm{log}^{\mathbb{E}[e^{\theta( f_{\mathcal{Z}}(\omega_{i}^{\mathrm{iid}}))}]}+\theta F_{f,\mathcal{D}}(\mathbf{z}))\} \notag\\
&\quad\quad\quad+ \exp\{-s\sup_{\theta<0}(-\theta \epsilon - \mathrm{log}^{\mathbb{E}[e^{\theta( f_{\mathcal{Z}}(\omega_{i}^{\mathrm{iid}}))}]}+\theta F_{f,\mathcal{D}}(\mathbf{z}))\} \notag\\
&= \exp\{-s\sup_{\theta>0}( \theta (F_{f,\mathcal{D}}(\mathbf{z})+\epsilon) - \mathrm{log}^{\mathbb{E}[e^{\theta( f_{\mathcal{Z}}(\omega_{i}^{\mathrm{iid}}))}]})\} \notag\\ 
&\quad\quad\quad+ \exp\{-s\sup_{\theta<0}(-\theta (F_{f,\mathcal{D}}(\mathbf{z})-\epsilon) - \mathrm{log}^{\mathbb{E}[e^{\theta( f_{\mathcal{Z}}(\omega_{i}^{\mathrm{iid}}))}]})\} \notag\\
&= \exp\{-s(\mathcal{L}_{X}(F_{f,\mathcal{D}}(\mathbf{z}))+\epsilon)\}+\exp\{-s(\mathcal{L}_{X}(F_{f,\mathcal{D}}(\mathbf{z}))-\epsilon)\}
\end{align}

where $\mathcal{L}_{X}(a) = \sup_{\theta>0} \log(\frac{e^{\theta a}}{M_{X}(\theta)})$ if $a > \mathbb{E}[X]$ and $\mathcal{L}_{X}(a) = \sup_{\theta<0} \log(\frac{e^{\theta a}}{M_{X}(\theta)})$ if $a < \mathbb{E}[X]$.

The proof for the superiority of orthogonal estimator is similar as above, and we include it here for completeness. 

For $\lambda>0,\epsilon \in\mathbb{R}$, we have:

\begin{align}
\mathbb{P}[\widehat{F}^{\mathrm{ort}}_{f,\mathcal{D}}(\mathbf{z})-F_{f,\mathcal{D}}(\mathbf{z}) \geq \epsilon] &\leq 
\exp\{-\sup_{\theta>0}(\lambda \epsilon - \mathrm{log}^{\mathbb{E}[e^{\lambda(\widehat{F}^{\mathrm{ort}}_{f,\mathcal{D}}(\mathbf{z})-F_{f,\mathcal{D}}(\mathbf{z}))}]})\} \notag\\
&= \exp\{-\sup_{\theta>0}(\lambda (\epsilon+F_{f,\mathcal{D}}(\mathbf{z})) - \mathrm{log}^{\mathbb{E}[e^{\lambda(\widehat{F}^{\mathrm{ort}}_{f,\mathcal{D}}(\mathbf{z}))}]})\}
\end{align}

We can derive similarly such probability bound for iid estimator. Then by applying Corollary \ref{main_cor}, we know that $\mathrm{log}^{\mathbb{E}[e^{\lambda \widehat{F}^{\mathrm{ort}}_{f,\mathcal{D}}(\mathbf{z})}]} \leq \mathrm{log}^{\mathbb{E}[e^{\lambda \widehat{F}^{\mathrm{iid}}_{f,\mathcal{D}}(\mathbf{z})}]}$. With such relationship, we know directly that orthogonal estimator has better upper bound than iid estimator.

The same follows for $\lambda <0$. And we can combine these two cases and derive that  $\mathbb{P}[|\widehat{F}^{\mathrm{ort}}_{f,\mathcal{D}}(\mathbf{z})-F_{f,\mathcal{D}}(\mathbf{z})| \geq \epsilon]$ has better bound than $\mathbb{P}[|\widehat{F}^{\mathrm{iid}}_{f,\mathcal{D}}(\mathbf{z})-F_{f,\mathcal{D}}(\mathbf{z})| \geq \epsilon]$.

Finally, for the MSE of the iid estimator, we know from the independence of $(\omega_i)_{i=1}^{s}$ that:
\begin{equation}
\mathrm{MSE}(\widehat{F}^{\mathrm{iid}}_{f,\mathcal{D}}(\mathbf{z}))=
\frac{1}{s^2} \sum_{i=1}^s \mathrm{Var}[f(\omega_i^{\top}\mathbf{z})] = 
\frac{1}{s} \mathrm{Var}[f(\omega_1^{\top}\mathbf{z})]
\end{equation}
We can also decompose the MSE of orthogonal estimator as:
\begin{equation}
\mathrm{MSE}(\widehat{F}^{\mathrm{ort}}_{f,\mathcal{D}}(\mathbf{z}))=
\frac{1}{s} \mathrm{Var}[f(\omega_1^{\top}\mathbf{z})] + 
\frac{1}{s^2} \sum_{i\neq j}(\mathbb{E}[f(\omega_i^{\top}\mathbf{z})f(\omega_j^{\top}\mathbf{z})] - 
\mathbb{E}[f(\omega_i^{\top}\mathbf{z})] \mathbb{E}[f(\omega_j^{\top}\mathbf{z})] ) 
\end{equation}
Since a subset of two ND variables are also ND, the ND of $(f(\omega_i^{\top}\mathbf{z}))_{i=1}^s$ implies that the second part of $\mathrm{MSE}(\widehat{F}^{\mathrm{ort}}_{f,\mathcal{D}}(\mathbf{z}))$ is negative, which completes the proof.

\end{proof}

We further notice that since $\mathcal{D}$ is an isotropic probabilistic distribution on $\mathbb{R}^d$ which is rotation invariant, then for $\omega \sim \mathcal{D}$ and an odd function $\mathrm{odd}[f]$, we have $\mathbb{P}(\omega)=\mathbb{P}(-\omega)$ and $\mathrm{odd}[f](\omega^{\top}\mathbf{z})=-\mathrm{odd}[f](-\omega^{\top}\mathbf{z})$. Therefore, 
\begin{equation}
F_{\mathrm{odd}[f],\mathcal{D}}=\mathbb{E}_{\omega \sim \mathcal{D}}[\mathrm{odd}[f](\omega^{\top}\mathbf{z})]=\int_{\mathcal{D}}\mathrm{odd}[f](\omega^{\top}\mathbf{z})d\mathbb{P}(\omega)=0
\end{equation}

\subsection{Proof of Theorem \ref{Thm_StrongConcentration_main}}

\SecondTheorem*

\begin{proof}

Firstly, we decompose the estimator into increasing and decreasing parts as stated in \textbf{F2}: 
\begin{align}
\widehat{F}^{\mathrm{ort}}_{f,\mathcal{D}}(\mathbf{z}) &= \frac{1}{s} \sum_{i=1}^s f(|{\omega_i^\mathrm{ort}}^{\top} \mathbf{z}|) = \frac{1}{s}
\sum_{i=1}^s f^{+}(|{\omega_i^\mathrm{ort}}^{\top} \mathbf{z}|) + \frac{1}{s}
\sum_{i=1}^s f^{-}(|{\omega_i^{ort}}^{\top} \mathbf{z}|) \notag\\
&\overset{\mathrm{def}}{=} \widehat{F}^{\mathrm{ort},+}_{f,\mathcal{D}}(\mathbf{z}) + \widehat{F}^{\mathrm{ort},-}_{f,\mathcal{D}}(\mathbf{z})    
\end{align}
which are ND respectively. 

For bounded function $f$, we can apply Chernoff–Hoeffding inequalities for ND random variables \cite{Dub1998} and have:
\begin{equation}
\mathbb{P}[| \widehat{F}^{\mathrm{ort},+}_{f,\mathcal{D}}(\mathbf{z}) - F^{+}_{f,\mathcal{D}}(\mathbf{z})| \geq \epsilon]\leq 
2\exp(-\frac{2 s \epsilon^2}{ {(b^{+}-a^{+})}^2})    
\end{equation}
\begin{equation}
\mathbb{P}[|\widehat{F}^{\mathrm{ort},-}_{f,\mathcal{D}}(\mathbf{z}) - F^{-}_{f,\mathcal{D}}(\mathbf{z})| \geq \epsilon] \leq 
2\exp(-\frac{2 s \epsilon^2}{ {(b^{-}-a^{-})}^2})    
\end{equation}
 
Therefore, 
\begin{align}
&\quad\ \mathbb{P}[|\widehat{F}^{\mathrm{ort}}_{f,\mathcal{D}}(\mathbf{z})-F_{f,\mathcal{D}}(\mathbf{z})| \geq \epsilon] \notag\\
&= \mathbb{P}[|\widehat{F}^{\mathrm{ort},+}_{f,\mathcal{D}}(\mathbf{z}) + \widehat{F}^{\mathrm{ort},-}_{f,\mathcal{D}}(\mathbf{z}) - F^{+}_{f,\mathcal{D}}(\mathbf{z}) - F^{-}_{f,\mathcal{D}}(\mathbf{z})| \geq \epsilon]  \notag\\
&\leq \mathbb{P}[|\widehat{F}^{\mathrm{ort},+}_{f,\mathcal{D}}(\mathbf{z}) - F^{+}_{f,\mathcal{D}}(\mathbf{z})| + |\widehat{F}^{\mathrm{ort},-}_{f,\mathcal{D}}(\mathbf{z}) - F^{-}_{f,\mathcal{D}}(\mathbf{z})| \geq \epsilon]  \notag\\
&\leq \mathbb{P}[|\widehat{F}^{\mathrm{ort},+}_{f,\mathcal{D}}(\mathbf{z}) - F^{+}_{f,\mathcal{D}}(\mathbf{z})|\geq \frac{\epsilon}{2}] + \mathbb{P}[|\widehat{F}^{\mathrm{ort},-}_{f,\mathcal{D}}(\mathbf{z}) - F^{-}_{f,\mathcal{D}}(\mathbf{z})|\geq \frac{\epsilon}{2}]  \notag\\
&\leq 2\exp(-\frac{s \epsilon^2}{2 {(b^{+}-a^{+})}^2})+2\exp(-\frac{s \epsilon^2}{2{(b^{-}-a^{-})}^2})    
\end{align}

This procedure can be adapted to the case of unbounded function $f$ with similar steps in Theorem \ref{first_thm}, so we skip it.

\end{proof}

\subsection{Proof of Theorem \ref{uniformtheory}}
\label{conv}

\ThirdTheorem*

\begin{proof}
Motivated by \cite{rahimi}, the uniform convergence for OMCs can be proved in the following way. Define $g(\mathbf{z}) = \widehat{F}^{\mathrm{ort}}_{f,\mathcal{D}}(\mathbf{z}) - F_{f,\mathcal{D}}(\mathbf{z})$. Given the definition of $\widehat{F}^{\mathrm{ort}}_{f,\mathcal{D}}(\mathbf{z})$, it is unbiased, i.e. $\mathbb{E}[g(\mathbf{z})] = \mathbb{E}[\widehat{F}^{\mathrm{ort}}_{f,\mathcal{D}}(\mathbf{z}) - F_{f,\mathcal{D}}(\mathbf{z})] = 0$. 

Let $\mathcal{M} \subseteq \mathbb{R}^{d}$ be compact with diameter $\mathrm{diam}(\mathcal{M})$ and $\mathbf{z}\in \mathcal{M}$. We can find a $\epsilon$-net such that it can covers $\mathcal{M}$ with at most $P = (\frac{4\mathrm{diam}(\mathcal{M})}{r})^{d}$ balls of radius $r$. Denote $\mathbf\{{z_i\}}_{i=1}^{P}$ as the centers of the these balls. If $|g(\mathbf{z}_i)| < \frac{\epsilon}{2}$ and Lipschitz constant $L_{g}$ of $g$ satisfies: $L_g < \frac{\epsilon}{2r}, \forall i\in[P]$, then $|g(\mathbf{z})| < \epsilon$. By applying the union bound followed by Hoeffding's inequality applied to the anchors in the $\epsilon$-net, we can have the following:
\begin{equation}
    \mathbb{P}[\bigcup_{i=1}^{P}|g(\mathbf{z}_i)|\geq \frac{\epsilon}{2}] \leq P\cdot p(\frac{\epsilon}{2})
\end{equation}
If $f$ is differentiable, $L_g = \max_{\mathbf{z} \in\mathcal{M}}||\nabla g(\mathbf{z^{*}})||$. From the linearity of expectation, we can have $\mathbb{E}[\nabla \widehat{F}^{\mathrm{ort}}_{f,\mathcal{D}}(\mathbf{z})] =  \nabla F_{f,\mathcal{D}}(\mathbf{z})$, therefore we can have:
\begin{align}
    \mathbb{E}[L_{g}^{2}] &= \mathbb{E}[||\nabla \widehat{F}^{\mathrm{ort}}_{f,\mathcal{D}}(\mathbf{z^{*}}) - \nabla F_{f,\mathcal{D}}(\mathbf{z^{*}})||^{2}]\notag\\
    &=\mathbb{E}[||\nabla \widehat{F}^{\mathrm{ort}}_{f,\mathcal{D}}(\mathbf{z^{*}})||^{2} + ||\nabla F_{f,\mathcal{D}}(\mathbf{z^{*}})||^{2} - 2 {\nabla \widehat{F}^{\mathrm{ort}}_{f,\mathcal{D}}(\mathbf{z^{*}})}^{T}\nabla F_{f,\mathcal{D}}(\mathbf{z^{*}})]\notag\\
    &= \mathbb{E}[||\nabla \widehat{F}^{\mathrm{ort}}_{f,\mathcal{D}}(\mathbf{z^{*}})||^{2}] + \mathbb{E}[||\nabla F_{f,\mathcal{D}}(\mathbf{z^{*}})||^{2}] - 2\mathbb{E}[||\nabla F_{f,\mathcal{D}}(\mathbf{z^{*}})||^{2}]\notag\\
    &= \mathbb{E}[||\nabla \widehat{F}^{\mathrm{ort}}_{f,\mathcal{D}}(\mathbf{z^{*}})||^{2}] - \mathbb{E}[||\nabla F_{f,\mathcal{D}}(\mathbf{z^{*}})||^{2}]
\end{align}
Therefore, $\mathbb{E}[L_{g}^{2}] \leq \mathbb{E}[||\nabla \widehat{F}^{\mathrm{ort}}_{f,\mathcal{D}}(z^{*})||^{2}] \leq \mathbb{E}_{\mathcal{D}}[||\omega L_{f}||^{2}] = \sigma^{2}{L_{f}}^{2}$. Finally, if $f$ is not differentiable, we can obtain exactly the same bound via standard finite-difference analysis.
According to the Markov Inequality, we have the following:
\begin{equation}
    \mathbb{P}[L_g\geq \frac{\epsilon}{2r}] \leq (\frac{2r\sigma L_{f}}{\epsilon})^{2}.
\end{equation}
Thus, by union bound, we can conclude that:
\begin{equation}
    \mathbb{P}[\sup_{\mathbf{z}\in \mathcal{M}} |g(\mathbf{z})| \geq \epsilon]\leq(\frac{4\mathrm{diam}(\mathcal{M})}{r})^{d}\cdot p(\frac{\epsilon}{2}) + (\frac{2r\sigma L_{f}}{\epsilon})^{2}, 
\end{equation}
which is our results for general $f$. Now let us consider the case when $f$ is bounded.
For the case of \textbf{F1}, we can let $p(\epsilon) = 2\exp(-{\frac{2s\epsilon^{2}}{(b-a)^{2}}})$. Then:
\begin{equation}
     \mathbb{P}[\sup_{\mathbf{z}\in \mathcal{M}}|g(\mathbf{z})| \geq \epsilon]\leq2(\frac{4\mathrm{diam}(\mathcal{M})}{r})^{d}\exp(-{\frac{s\epsilon^{2}}{2(b-a)^{2}}}) + (\frac{2r\sigma L_{f}}{\epsilon})^{2} 
\end{equation}
For the case of \textbf{F2/F3}, we can have: $p(\epsilon) =
2(\exp(-\frac{s \epsilon^2}{2{(b^{+}-a^{+})}^2})+\exp(-\frac{s \epsilon^2}{2{(b^{-}-a^{-})}^2}))$. Then:
\begin{equation}
    \mathbb{P}[\sup_{\mathbf{z}\in \mathcal{M}}|g(\mathbf{z})| \geq \epsilon]\leq2(\frac{4\mathrm{diam}(\mathcal{M})}{r})^{d}[\exp(-\frac{s \epsilon^2}{8 {(b^{+}-a^{+})}^2})+\exp(-\frac{s \epsilon^2}{8 {(b^{-}-a^{-})}^2})] + (\frac{2r\sigma L_{f}}{\epsilon})^{2} 
\end{equation}
One can take $C=2\cdot 4^{d}$ here. In order to find smallest $s$ such that \textbf{F1/F2/F3} can satisfy this bound, we can optimize for $r$ and this is how we get the asymptotic value of the number of samples $s$ that provides $\epsilon$-accuracy (we assume here that bounds on $f$ are constants):
\begin{equation}
    s = {\Theta}(\frac{d}{\epsilon^{2}}\log(\frac{\sigma L_{f} \mathrm{diam}(\mathcal{M})}{\epsilon})).
\end{equation}
Another case is that $f$ is unbounded, For the case of $\textbf{F1}$, we can let $p(\epsilon) = \exp\{-s(\mathcal{L}_{X}(F_{f,\mathcal{D}}(\mathbf{z}))+\epsilon)\}+\exp\{-s(\mathcal{L}_{X}(F_{f,\mathcal{D}}(\mathbf{z}))-\epsilon)\}$. Then:
\begin{align}
    \mathbb{P}[\sup_{\mathbf{z}\in \mathcal{M}}|g(\mathbf{z})| &\geq \epsilon]\leq (\frac{4\mathrm{diam}(\mathcal{M})}{r})^{d}\cdot(\exp\{-s(\mathcal{L}_{X}(F_{f,\mathcal{D}}(\mathbf{z}))+\frac{\epsilon}{2})\}\notag\\
    &\quad\quad\quad+\exp\{-s(\mathcal{L}_{X}(F_{f,\mathcal{D}}(\mathbf{z}))-\frac{\epsilon}{2})\})+ (\frac{2r\sigma L_{f}}{\epsilon})^{2}
\end{align}
For the case of $\textbf{F2/F3}$, we can have: $p(\epsilon) =  \exp(-s\mathcal{L}_{X^{+}}(F_{f,\mathcal{D}}(\mathbf{z})+\frac{\epsilon}{2})+\exp(-s\mathcal{L}_{X^{+}}(F_{f,\mathcal{D}}(\mathbf{z})-\frac{\epsilon}{2}) + \exp(-s\mathcal{L}_{X^{-}}(F_{f,\mathcal{D}}(\mathbf{z})+\frac{\epsilon}{2})+\exp(-s\mathcal{L}_{X^{-}}(F_{f,\mathcal{D}}(\mathbf{z})-\frac{\epsilon}{2}))$. Then:
\begin{align}
    \mathbb{P}[\sup_{\mathbf{z}\in \mathcal{M}}|g(\mathbf{z})| \geq \epsilon]&\leq (\frac{4\mathrm{diam}(\mathcal{M})}{r})^{d}\cdot[\exp(-s\mathcal{L}_{X^{+}}(F_{f,\mathcal{D}}(\mathbf{z})+\frac{\epsilon}{4}))\notag\\
    &\quad\quad\quad+\exp(-s\mathcal{L}_{X^{+}}(F_{f,\mathcal{D}}(\mathbf{z})-\frac{\epsilon}{4}))\notag\\
    &\quad\quad\quad+ \exp(-s\mathcal{L}_{X^{-}}(F_{f,\mathcal{D}}(\mathbf{z})+\frac{\epsilon}{4}))\notag\\
    &\quad\quad\quad+\exp(-s\mathcal{L}_{X^{-}}(F_{f,\mathcal{D}}(\mathbf{z})-\frac{\epsilon}{4}))] + (\frac{2r\sigma L_{f}}{\epsilon})^{2}
\end{align}
Still, one can take $C = 2\cdot 4^{d}$ here. In order to find smallest $s$ such that \textbf{F1/F2/F3} can satisfy the bound, we optimize for r and get the asymptotic value of the number of sample $s$ that provides $\epsilon$-accuracy(we assume $\mathcal{L}_{X}(F_{f,\mathcal{D}}(\mathbf{z}))$ or $ \mathcal{L}_{X^{+/-}}(F_{f,\mathcal{D}}(\mathbf{z}))$ mentioned in Theorem 1 and Theorem 2 are constants:
\begin{equation}
    s = \Theta(d\log(\frac{L_{f}\sigma(\mathrm{diam}(\mathcal{M}))}{\epsilon}))
\end{equation}
\end{proof}

\subsection{On the Uniform Convergence of OMCs for Improving OMC Kernel Ridge Regression Guarantees}
\label{ker}
Recalling the setting in Theorem 2 of \cite{psrnn}:\\\\
\begin{theorem}
Assume that a dataset $\mathcal{X} = \{x_1, x_2,...,x_n\}$ is taken from a ball $\mathcal{B}$ of a fixed radius $r$ which is independent to the dimensionality of the data $n$, and size of dataset $N$, and the center $x_0$.\\\\
Consider kernel ridge regression adopting a smooth RBF kernel, especially Gaussian kernel. Let $\widehat{\Delta}_{iid}$ denote the smallest positive number such that $\mathbf{\widehat{K}}_{iid} + \lambda N\mathbf{I}_{N}$ is a $\Delta$-approximation of $\mathbf{K}+\lambda N \mathbf{I}_N$, where $\mathbf{\widehat{K}}_{iid}$ is an approximate kernel matrix obtained by using unstructured random features. Then for any $a > 0$, 
\begin{equation}
    \mathbb{P}[\widehat{\Delta}_{iid} > a]\leq p_{N,m}^{iid}(\frac{a\sigma_{min}}{N}),
\end{equation}
where $p_{N,m}^{iid} = N^{2}e^{-Cmx^{2}}$ for some universal constant $C>0$, $m$ is the number of random features used, $\sigma_{min}$ is the smallest singular value of $\mathbf{\widehat{K}}+\lambda N \mathbf{I}_N$ and $N$ is the dataset size. If instead orthogonal random features are used then for the corresponding spectral parameter $\widehat{\Delta}_{ort}$ the following holds:
\begin{equation}
    \mathbb{P}[\widehat{\Delta}_{ort}>a]\leq p_{N,m}^{ort}(\frac{a\sigma_{min}}{N}),
\end{equation}
where function $p_{N,m}^{ort}$ satisfies: $p_{N,m}^{ort}<p_{N,m}^{iid}$, for $n$ large enough.\\
\end{theorem}
Based on this original version, we would like to offer a refined version as the following:\\\\
Rather than having $p_{N,m}^{iid} = N^{2}e^{-Cmx^{2}}$, we can further remove $N^{2}$ by exploiting uniform convergence property if $z = x_i - x_j$ is in a compact set, $x_i, x_j$ are arbitrary two datapoints in the dataset, meaning that \begin{equation}
    p_{N,m}^{iid} = e^{-Cmx^{2}}
\end{equation}
Following the same logic, we can still have $p_{N,m}^{ort}<p_{N,m}^{iid}$, for $n$ large enough, resulting in a much stronger guarantee for kernel ridge regression. Proof is the following:

\begin{proof}

Motivated by \cite{psrnn}, we can even substantially improve theoretical guarantees offered in its Theorem 2 with the uniform convergence property that we derived above. In order to achieve it, we will improve the Lemma 1 of \cite{psrnn}. We discuss all steps in detail below.

For an RBF kernel $\mathbf{K}: \mathbb{R}^{n}\times\mathbb{R}^{n}$, with a corresponding random feature map: $\Phi_{m.n}: \mathbb{R}^{n}\rightarrow \mathbb{R}^{2m}$, we can approximate it with a randomized kernel estimator $\widehat{\mathbf{K}}$. Assume that for any $i,j \in [N]$, the following holds for any $c > 0:\mathbb{P}[|\Phi_{m,n}(x_i)^{T}\Phi_{m,n}(x_j) - \mathbf{K}(x_i, x_j)|>c]\leq g(c)$ for some fixed function $g:\mathbb{R}\rightarrow\mathbb{R}.$ Then with probability at least $1 - g(c),$ matrix $\widehat{\mathbf{K}}+\lambda \mathbf{I}_N$ is a $\Delta$-spectral approximation of matrix ${\mathbf{K}}+\lambda \mathbf{I}_N$ for $\Delta = \frac{Nc}{\sigma_{min}}$, where $\sigma_{min}$ stands for the minimal singular value of ${\mathbf{K}}+\lambda \mathbf{I}_N$.

Denote $\mathbf{K} + \lambda N\mathbf{I}_N = \mathbf{V}^T \mathbf{\Sigma}^{2}\mathbf{V}$, where an orthogonal matrix $\mathbf{V} \in \mathbb{R}^{N\times N}$ and a diagonal matrix $\Sigma\in\mathbf{R}^{N\times N}$ define the eigendecomposition of $\mathbf{K+\lambda N\mathbf{I}_N}$. As shown in the paper, in order to prove that $\widehat{\mathbf{K}} + \lambda N\mathbf{I}_N$ is a $\Delta$-spectral approximation of $\mathbf{K} + \lambda N\mathbf{I}_N$, it suffices to show that:
\begin{equation}
    ||\mathbf{\Sigma}^{-1} \mathbf{V} \widehat{\mathbf{K}}\mathbf{V}^{T}\mathbf{\Sigma}^{-1}- \mathbf{\Sigma}^{-1} \mathbf{V} \mathbf{K}\mathbf{V}^{T}\mathbf{\Sigma}^{-1}||_{2}\leq \Delta
\end{equation}
With the definition of $l_2$ norm and Frobenius norm, we can have:
\begin{align}
        &\quad \mathbb P[{||\mathbf{\Sigma}^{-1} \mathbf{V} \widehat{\mathbf{K}}\mathbf{V}^{T}\mathbf{\Sigma}^{-1}- \mathbf{\Sigma}^{-1} \mathbf{V} \mathbf{K}\mathbf{V}^{T}\mathbf{\Sigma}^{-1}||_{2} > \Delta}]\notag\\
        &\leq \mathbb P[{||\mathbf{\Sigma}^{-1} \mathbf{V} ||\widehat{\mathbf{K}}-\mathbf{K}||_{F}\mathbf{V}^{T}\mathbf{\Sigma}^{-1}||_{2} > \Delta}]\notag\\
    &= \mathbb P[||\mathbf{\widehat{K}} - \mathbf{K}||_{F}^{2} > \frac{\Delta^{2}}{||\mathbf{\Sigma}^{-1} \mathbf{V}||_{2}^{2}\cdot||\mathbf{V}^{T}\mathbf{\Sigma}^{-1}||_{2}^{2}}]\notag\\
    &\leq \mathbb P[||\mathbf{\widehat{K}} - \mathbf{K}||_{F}^{2} > \Delta^{2}\sigma_{min}^{2}].
\end{align}
The last inequality we use the fact that $||\mathbf{\Sigma}^{-1} \mathbf{V}||_2^{2}\leq\frac{1}{\sigma_{min}}$ and $||\mathbf{V}^{T}\mathbf{\Sigma}^{-1}||_2^{2}\leq\frac{1}{\sigma_{min}}$ because $\mathbf{V}$ is an isometric matrix.\\
Most importantly, we can refine the proof of lemma 1 in \cite{psrnn} using the uniform convergence property, provided that $z = x_i - x_j$ is in a compact set. Then the following inequalities hold:
\begin{align}
    &\quad\mathbb P[||\mathbf{\widehat{K}} - \mathbf{K}||_{F}^{2} > \frac{\Delta^{2}}{||\mathbf{\Sigma}^{-1} \mathbf{V}||_{2}^{2}\cdot||\mathbf{V}^{T}\mathbf{\Sigma}^{-1}||_{2}^{2}}]\notag\\
    &\leq \mathbb{P}[|\widehat{\mathbf{K}}_{i,j} - \mathbf{K}_{i,j}| > \frac{\Delta\sigma_{min}}{N}]\notag\\
    &=\mathbb{P}[|\Phi_{m,n}(x_i)^{T}\Phi_{m,n}(x_j) - \mathbf{K}_{i,j}| > \frac{\Delta\sigma_{min}}{N}]
\end{align}
Therefore, the probability that $\mathbf{\widehat{K}} + \lambda\mathbf{I}_N$ is a $\Delta$-spectral approximation of $\mathbf{K} + \lambda\mathbf{I}_N$ is at least $1-g(c)$ for $c = \frac{\Delta\sigma_{min}}{N}$. Afterwards, we can prove lemma 2,3,4,5 in \cite{psrnn} in the same way. Therefore, it shows that we can have a stronger concentration result for kernel ridge regression.
\end{proof}
\section*{APPENDIX B: Demystifying Orthogonal Monte Carlo and Beyond - Experiments}
\label{app:exp}

In our experiment with the particle algorithm ($\mathrm{opt}$-NOMC), we use: $\eta=1.0, \delta=0.1, T=50000$.  

\subsection{Clock Time Comparison for NOMC} \label{appendix:clocktime}

In order to present the efficiency of our NOMC optimization procedure, we run our algorithm on a \textbf{single 6-core computer with Intel Core i7 CPU}, and parameter $d$ range from 8 to 256. The algorithm here uses the plain gradient descent method for optimization. Please note that this wall clock time is just a \textbf{one-time cost}, even if a new ensemble of samples is required at each iteration of the higher-level algorithm. In such a case that one-time optimized ensemble is simply randomly rotated using independently chosen random rotations, as mentioned in main text.
Furthermore, we can always improve the efficiency by multi-machine parallelization, which however is not the focus of this work.

\small
\begin{figure}[h] 
    \begin{minipage}{1.0\textwidth}
    \includegraphics[width=.99\linewidth]{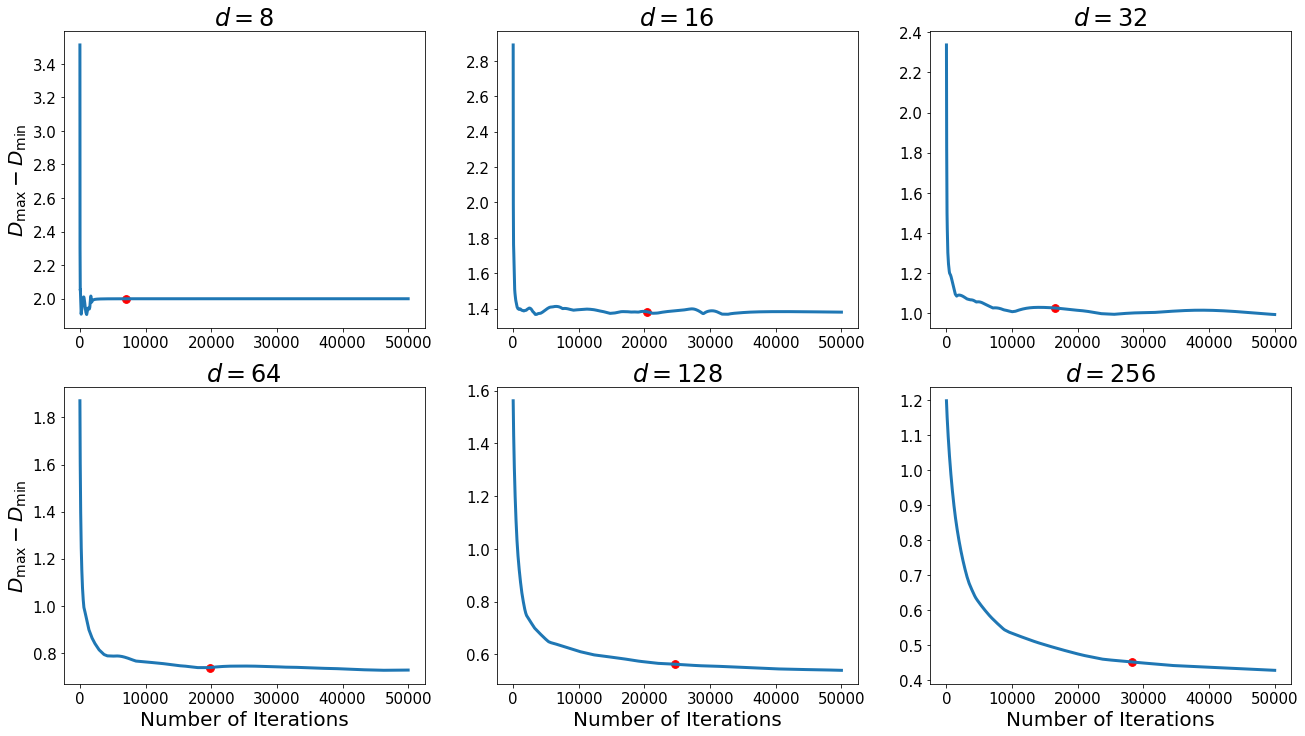}
    \end{minipage}
    \vspace{-4.5mm}
    \caption{Clock time comparisons with $d=8,16,32,64,128,256$, and $s=5d$. $D_{\mathrm{max}}$ is the maximum distance among all the points on the unit-sphere, and $D_{\mathrm{min}}$ represents the minimum of them. We use the difference between $D_{\mathrm{max}}$ and $D_{\mathrm{min}}$ as the y-axis, which should gradually decrease with each iterations. Besides, the red point in each line represents the first position where the absolute change in y-axis within the past 5000 iterations is below 0.01. We set parameters $\delta=0.1$ and $\eta=1$ in Algorithm \ref{algorithm_NOMC}. }
\label{fig:clocktime}
\end{figure}
\normalsize

\small
\vspace{-1mm}
\begin{center}
\label{table:clocktime}
 \begin{tabular}{||c | c | c | c | c | c | c ||} 
 \hline
 $d$ & 8 & 16 & 32 & 64 & 128 & 256  \\
 \hline
 Clock Time  & 20 seconds & 2 minutes & 5 minutes & 22 minutes & 2 hours & 14 hours\\
 \hline
\end{tabular}
\end{center}
\vspace{-1mm}
\small Table \ref{table:clocktime}: Clock time comparison for different $d$. The time here represents the first time when the absolute change in $D_{\mathrm{max}} - D_{\mathrm{min}}$ within the past 5000 iterations is below 0.01 (same as the red point in Figure \ref{fig:clocktime}).
\normalsize

\subsection{Experimental Details for Kernel Approximation Experiment}
In Section 5, we present a result showing that our NOMC method indeed outperforms other algorithms. Specifically, we adopt mean squared error (MSE) as the error measure for pointwise estimation. As for the data set, rather than using theoretically simulated data, we adopted a variety of the data set from the UCI Machine Learning Repository for our experiments. Due to space constraints, we only select one of the experimental results from those data set, which is Letter Recognition Data Set. Also, this is one of the most popular and classical experimental data set. In our experiment, we compared 8 different kernels, which is shown in the table 6.7 below. For each kernel, we tested the performance of MC, QMC, B-OMC, NOMC for 10 multipliers, ranging from 1 to 10. For each multiplier, we performed 450 pointwise estimations to 100 randomly sampled data pairs and calculated the average of the MSE, in order to relieve the impact of single selection bias. Empirically, for the purpose of ensuring the kernel values in an appropriate range, we scaled the dataset using the mean distance of the 50th $l_{2}$ nearest neighbor for 1000 sampled datapoints.(see:\cite{ort})\\

This experiment is implemented in {\fontfamily{qcr}\selectfont Python 3.7} and executed on a standard {\fontfamily{qcr}\selectfont 1.7 GHz Dual-Core Intel Core i7}.

\small
\vspace{-1mm}
\begin{center}
\label{table:kernel_table}
 \begin{tabular}{||c | c| c |  c ||} 
 \hline
 $\textrm{Kernel name}$ & $\textrm{Kernel function}$ & $\textrm{Function $\phi$}$ &  $\textrm{Fourier density}$ \\ [0.5ex] 
 \hline\hline
 $\textrm{Gaussian}$ &
 $\sigma^{2}\exp{(-\frac{1}{2\lambda^{2}})z^{2}}$ & $\cos(\omega^{T}\mathbf{x}+b)$ &  $\frac{\sigma^{2}}{(2\pi\lambda^{2})^{n/ 2}}\exp{-\frac{1}{2\lambda^{2}}||w||_{2}^{2}}$ \\ 
  \hline  
  $\textrm{Matérn}$\cite{geom} & $\sigma^{2}\frac{2^{1-\nu}}{\Gamma(\nu)}(\sqrt{2\nu}z)^{\nu}K_{\nu}(\sqrt{2\nu}z)$ & $\cos(\omega^{T}\mathbf{x}+
  b)$ & $\frac{\Gamma(\nu + n/2)}{\Gamma(\nu)(2\nu\pi)^{n/2}}(1 + \frac{||w||^2}{2\nu})^{-\nu - p/2}$ \\ 
 \hline  
  $\textrm{Cauchy}$\cite{rahimi} & $\Pi_{d}\frac{2}{1+z_{d}^{2}}$ & $\cos(\omega^{T}\mathbf{x}+b)$ &  $e^{-||\omega||_{1}}$ \\ 
 \hline  
 $\textrm{Angular}$ & $1 - \frac{2\theta_{\mathbf{x}, \mathbf{y}}}{\pi}$ & $\mathrm{sgn}(\omega^{T}x)$ & N/A \\ 
 \hline  
 $\textrm{Quadratic}$ & $\mathbb{E}_{\omega}[\phi(\mathbf{x})\phi(\mathbf{y})]$ & $(\omega^{T}x)^{2}$ & N/A \\ 
 \hline  
 $\textrm{Tanh}$ & $\mathbb{E}_{\omega}[\phi(\mathbf{x})\phi(\mathbf{y})]$ & $\mathrm{tanh}(\omega^{T}x)$ & N/A \\ 
 \hline  
 $\textrm{Sine}$ & $\mathbb{E}_{\omega}[\phi(\mathbf{x})\phi(\mathbf{y})]$ & $\sin(\omega^{T}x)$ & N/A \\ 
 \hline 
\end{tabular}
\end{center}
\vspace{-1mm}
\small Table \ref{table:kernel_table} : Tested kernels, their corresponding kernel functions (we give compact form if it exists), mappings $\phi$ such that $K(\mathbf{x},\mathbf{y})=\mathbb{E}_{\omega}[\phi(\mathbf{x})\phi(\mathbf{y})]$ (used in MC sampling), and Fourier denssities (valid only for hift-invariant kernels). For Matérn kernel, $\Gamma (\cdot)$ denotes the gamma function, $K_{\nu}(\cdot)$ denotes the modified Bessel function of the second kind, and $\nu$ is a non-negative parameter. Parameter $\lambda$ denotes standard deviation, $\mathbf{z} =(z_{1},...,z_{d})^{\top} = \mathbf{x}-\mathbf{y}$, $z=\|\mathbf{z}\|_{2}$ and $b \sim \mathrm{Unif}[0, 2\pi]$.
\normalsize

\subsection{Experimental Details for Sliced Wasserstein Distance Experiment} 
\label{appendix:swd}

We run these experiments on a single 6-core computer with Intel Core i7 CPU.
For the Sliced Wasserstein Distance experiments in Section 5, we use the same procedure as in the kernel approximation experiments and tested on 8 classes of distributions. For each class, we have two multivariate distributions with different means and covariance matrices. Following the formula in Equation (\ref{equation:swd}), we replaced the iid samples $\mathbf{u} \sim \mathrm{Unif}(\mathcal{S}^{d-1})$ (which is the plain MC method) with samples from multiple orthogonal blocks, near orthogonal algorithms, and Halton sequences (which are B-OMC, NOMC and QMC respectively). We independently sample 100 thousands data points from each of the two distributions from the same class, and then compute the projections on the directions of $\mathbf{u}$. The specific details regarding mean and covariance matrix of each distribution are in Table \ref{table:swd_table}. Let $\mathbf{A}$ be a $d \times d$ matrix with each entry generated from standard univariate gaussian distribution.
Also let $\mathbf{D}$ be a $d \times d$ matrix obtained from the distribution of $\mathbf{A}$ by zeroing all off-diagonal values to zero. We take $\mathbf{M} \overset{\mathrm{def}}{=} \sqrt{d} \mathbf{A}^{\top}\mathbf{A}$ (note that $\mathbf{A}$ is positive semi-definite).

\small
\vspace{-1mm}
\begin{center}
\label{table:swd_table}
 \begin{tabular}{||c | c | c | c ||} 
 \hline
 $\textrm{Distribution name}$ & $\textrm{Mean}$ & $\textrm{Covariance Matrix}$ & $\textrm{Parameter}$ \\ [0.5ex] 
 \hline\hline
 $\textrm{Multivariate Gaussian}$ & $(0,0,...,0), (1,1,...,1)$ & $\mathbf{M}_1, \mathbf{M}_2$ & $\textrm{N/A}$ \\ 
  \hline  
  $\textrm{Multivariate T}$ & $(0,0,...,0), (1,1,...,1)$  & $\mathbf{M}_1, \mathbf{M}_2$ & $\textrm{df=10}$ \\ 
 \hline  
  $\textrm{Multivariate Cauchy}$ & $(0,0,...,0), (1,1,...,1)$ & $\mathbf{M}_1, \mathbf{M}_2$ & $\textrm{N/A}$ \\ 
 \hline  
 $\textrm{Multivariate Laplace}$ & $(0,0,...,0), (0,0,...,0)$ & $\mathbf{M}_1, \mathbf{M}_2$ & $\textrm{N/A}$ \\ 
 \hline  
 $\textrm{Gaussian Mixture Q=2}$ & $(0,...,0,1,...,1), (1,...,1,0,...,0)$ & $\mathbf{D}_1, \mathbf{D}_2$ & $\textrm{N/A}$\\ 
 \hline  
 $\textrm{Gaussian Mixture Q=3}$ & $(1,1,1,1,0,...,0), (0,...,0,1,1,1)$ & $\mathbf{D}_1, \mathbf{D}_2, \mathbf{D}_3$ & $\textrm{N/A}$\\ & $(0,...,0,1,1,1,0,...,0)$ & & \\
 \hline  
 $\textrm{Gaussian Mixture Q=4}$ & $(1,1,1,1,0,...,0), (0,0,1,1,0,...,0)$ & $\mathbf{D}_1, \mathbf{D}_2, \mathbf{D}_3, \mathbf{D}_4$ & $\textrm{N/A}$\\ & $(0,...,0,1,1,0,...,0), (0,...,0,1,1)$ & & \\
 \hline 
 $\textrm{Inverse Wishart}$ & $(0,0,...,0), (1,1,...,1)$ & $\mathbf{M}_1, \mathbf{M}_2$ & $\nu=10$ \\ 
 \hline
\end{tabular}
\end{center}
\vspace{-1mm}
\small Table \ref{table:swd_table} : Tested classes of distributions (from each we sampled two distributions for SWD computations), their corresponding means of modes, covariance matrices for different modes and other parameters (if applicable). 
\normalsize

\medskip

\small
\appendix

\end{document}